\documentclass[letterpaper]{article} 
\usepackage{aaai2026}  
\usepackage{times}  
\usepackage{helvet}  
\usepackage{courier}  
\usepackage[hyphens]{url}  
\usepackage{graphicx} 
\usepackage{natbib}  
\usepackage{caption} 
\usepackage{algorithm}

\usepackage{newfloat}
\usepackage{listings}
\DeclareCaptionStyle{ruled}{labelfont=normalfont,labelsep=colon,strut=off} 
\floatstyle{ruled}
\newfloat{listing}{tb}{lst}{}
\floatname{listing}{Listing}
\title{Semantic Volume: Quantifying and Detecting both External and Internal Uncertainty in LLMs}
\author{
    Xiaomin Li\textsuperscript{\rm 1}\thanks{Correspondence to: Xiaomin Li (xiaominli@g.harvard.edu).},
    Zhou Yu\textsuperscript{\rm 2},
    Ziji Zhang\textsuperscript{\rm 2},
    Yingying Zhuang\textsuperscript{\rm 2},
    Swair Shah\textsuperscript{\rm 2},\\
    Narayanan Sadagopan\textsuperscript{\rm 2},
    Anurag Beniwal\textsuperscript{\rm 2}
}
\affiliations{
    \textsuperscript{\rm 1}Harvard University\\
    \textsuperscript{\rm 2}Amazon
}

\usepackage{bibentry}

\usepackage{amsmath,amsfonts,bm}

\newcommand{\R}{\mathbb{R}}
\newcommand{\E}{\mathbb{E}}

\def\vu{{\bm{u}}}
\def\vv{{\bm{v}}}

\def\vx{{\bm{x}}}

\def\bvx{\bar{\bm{x}}}
\def\bvv{\bar{\bm{v}}}

\def\mA{{\bm{A}}}

\def\mI{{\bm{I}}}

\def\mM{{\bm{M}}}

\def\mP{{\bm{P}}}

\def\mS{{\bm{S}}}

\def\mV{{\bm{V}}}

\def\mX{{\bm{X}}}

\def\mSigma{{\bm{\Sigma}}}

\newcommand{\calD}{\mathcal{D}}
\newcommand{\calL}{\mathcal{L}}
\newcommand{\calM}{\mathcal{M}}
\newcommand{\calN}{\mathcal{N}}

\newcommand{\calH}{\mathcal{H}}

\newcommand{\bydef}{\stackrel{\text{def}}{=}}

\DeclareMathOperator{\Tr}{Tr}
\renewcommand{\P}{\mathbb{P}}

\usepackage[utf8]{inputenc} 
\usepackage[T1]{fontenc}    
\usepackage{url}            
\usepackage{booktabs}       
\usepackage{amsfonts}       
\usepackage{nicefrac}       
\usepackage{microtype}      
\usepackage{xcolor}         

\usepackage{multirow}
\usepackage{enumitem}
\usepackage[most]{tcolorbox}
\usepackage{float}
\usepackage{caption}
\usepackage{subcaption}
\usepackage{longtable}
\usepackage{listings}
\usepackage{amssymb}
\usepackage{mathtools}
\usepackage{amsthm}
\usepackage{multicol}
\usepackage{titletoc}
\usepackage{algpseudocode}
\usepackage{bm}   

\newtheorem{theorem}{Theorem}

\newtheorem{lemma}{Lemma}

\newtheorem{proposition}{Proposition}

\begin{document}

\maketitle

\begin{abstract}
Large language models (LLMs) have demonstrated remarkable performance across diverse tasks by encoding vast amounts of factual knowledge. However, they are still prone to hallucinations, generating incorrect or misleading information, often accompanied by high uncertainty. Existing methods for hallucination detection primarily focus on quantifying \textit{internal uncertainty}, which arises from missing or conflicting knowledge within the model. However, hallucinations can also stem from \textit{external uncertainty}, where ambiguous user queries lead to multiple possible interpretations. In this work, we introduce \textit{Semantic Volume}, a novel mathematical measure for quantifying both external and internal uncertainty in LLMs. Our approach perturbs queries and responses, embeds them in a semantic space, and computes the Gram matrix determinant of the embedding vectors, capturing their dispersion as a measure of uncertainty. Our framework provides a generalizable and unsupervised uncertainty detection method without requiring internal access to LLMs. We conduct extensive experiments on both external and internal uncertainty detections, demonstrating that our Semantic Volume method consistently outperforms existing baselines in both tasks. Additionally, we provide theoretical insights linking our measure to differential entropy, unifying and extending previous sampling-based uncertainty measures such as the semantic entropy. Semantic Volume is shown to be a robust and interpretable approach to improving the reliability of LLMs by systematically detecting uncertainty in both user queries and model responses.
\end{abstract}


\begin{links}
  \link{Code}{https://github.com/amazon-science/semantic-volume}
\end{links}

\section{Introduction}\label{sec:Introduction}
Large language models encode extensive knowledge from massive training data and have shown remarkable achievements on diverse tasks \citep{brown2020language, achiam2023gpt, touvron2023llama, llama3modelcard, anthropic_claude3_2023, guo2025deepseek, anil2023gemini}. Despite their success, LLMs still exhibit hallucination: generating information or conclusions that are incorrect, incomplete, fabricated, or misleading \citep{ji2023survey, huang2023survey, bang2023multitask, guerreiro2023hallucinations, chen2022towards, bastounis2024consistent}. These hallucinations can propagate false information, undermine decision-making, and damage the credibility of AI systems. Detecting the hallucination is a challenging task, and a growing stream of research leverages the uncertainty in LLMs for hallucination detection \citep{kuhn2023semantic, farquhar2024detecting, cole2023selectively, kadavath2022language, malinin2020uncertainty, fomicheva2020unsupervised, kossen2024semantic, lin2023generating, liu2024uncertainty, quevedo2024detecting}. Existing methods focus on \textbf{internal uncertainty}, which generally arises from missing relevant knowledge, conflicting information, or outdated data in the training corpus, and is assumed to reflect the model’s intrinsic limitations \citep{kuhn2023semantic, farquhar2024detecting, cole2023selectively, kossen2024semantic}. Nonetheless, such internal confusion, and consequently hallucinations, can also stem from \textbf{external uncertainty}, which occurs when the user's query is ambiguous, such as lacking context or having multiple possible interpretations due to typos, missing information, or ambiguous entities \citep{zhang2024clamber, min2020ambigqa, kuhn2022clam, kim2024aligning, chi2024clarinet, lee2023asking}. External uncertainty cases should be handled by requesting clarification from the user \citep{zhang2024clamber, min2020ambigqa, kuhn2022clam, lee2023asking}. For example, when asked the ambiguous question ``Who played Spiderman?", an LLM should ask the user to specify which movie they are referring to. It is important to note that internal uncertainty reflects true limitations of the model only after external uncertainty has been ruled out. To detect external uncertainty, current methods often rely on LLMs themselves to assess ambiguity via specialized prompting strategies \citep{kuhn2022clam, kim2024aligning, chi2024clarinet, zhang2024clamber}. In contrast, internal uncertainty (response uncertainty) detection follows two main paradigms: 1. \textbf{Probability-based} methods that utilize the token probabilities or entropy, requiring internal access to the model \citep{kadavath2022language, malinin2020uncertainty, quevedo2024detecting, ji2024llm}. 2. \textbf{Sampling-based} approaches, which sample multiple responses and propose measures to quantify uncertainty \citep{cole2023selectively, fomicheva2020unsupervised, zhang2023clarify, kuhn2023semantic, farquhar2024detecting, kossen2024semantic, lin2023generating}. A representative method in this category is the \textit{Semantic Entropy}, which clusters sampled responses into semantic equivalence classes and computes entropy across these clusters \citep{kuhn2023semantic, farquhar2024detecting}.

In this work, we introduce a unified sampling-based approach called the \textit{Semantic Volume}, which can generally be applied to detect both internal and external uncertainty in LLMs without requiring access to the model weights (see Figure~\ref{fig:pipeline}).  Our method generates perturbations of queries and responses, obtains their semantic embedding vectors,  and use a mathematical measure that essentially computes the determinant of the Gram matrix formed by the vectors, in order to quantify the semantic dispersion. Larger dispersion indicates higher uncertainty. More precisely, for external uncertainty, we prompt the LLM to generate multiple augmented versions of each query as perturbations, while for internal uncertainty, we sample multiple responses as perturbations. Then we take the normalized embedding vectors to be their representations. Putting these vectors as column vectors in a matrix $\mV$, the value $\det (\mV^\top \mV)$ mathematically measures the squared volume of the parallelepiped formed by these vectors. The log of this value essentially defines our Semantic Volume (see \ref{subsec:Definitions}). The idea is that if uncertainty is low, all perturbations should be similar or close to each other, resulting in a small dispersion (i.e., a smaller volume). 

We conduct comprehensive experiments on both query ambiguity detection and response uncertainty detection, demonstrating that our semantic volume method outperforms various baselines in both tasks. Additionally, we provide theoretical justification showing that our measure essentially captures the differential entropy of perturbation vectors, effectively quantifying overall semantic dispersion. Notably, the previously state-of-the-art \textit{Semantic Entropy} method emerges as a special case of our approach, highlighting the broader generalization of Semantic Volume over existing sampling-based methods. Our findings suggest that Semantic Volume provides a robust, interpretable framework for improving LLM reliability by systematically detecting and addressing uncertainty in both user queries and model responses. Below is a list of our main contributions:
\begin{itemize}
    \item We propose a novel mathematical measure for uncertainty detection, using the determinant of the Gram matrix (equivalently, parallelepiped volume) of embedding vectors to quantify semantic dispersion.
    \item Our method is \textbf{training-free} and \textbf{does not require internal access} to the model's hidden states or token probabilities).
    \item To the best of our knowledge, this is the \textbf{first framework to study both external and internal uncertainty} in LLMs.
    \item We validate our approach through comprehensive experiments on both external and internal uncertainty detection, demonstrating superior performance compared to various baselines.
    \item We provide \textbf{theoretical interpretations} of our semantic volume, linking it to differential entropy and generalizing existing sampling-based uncertainty measures.
\end{itemize}

\begin{figure*}[t!]
    \centering
   \includegraphics[width= 1.0\linewidth]{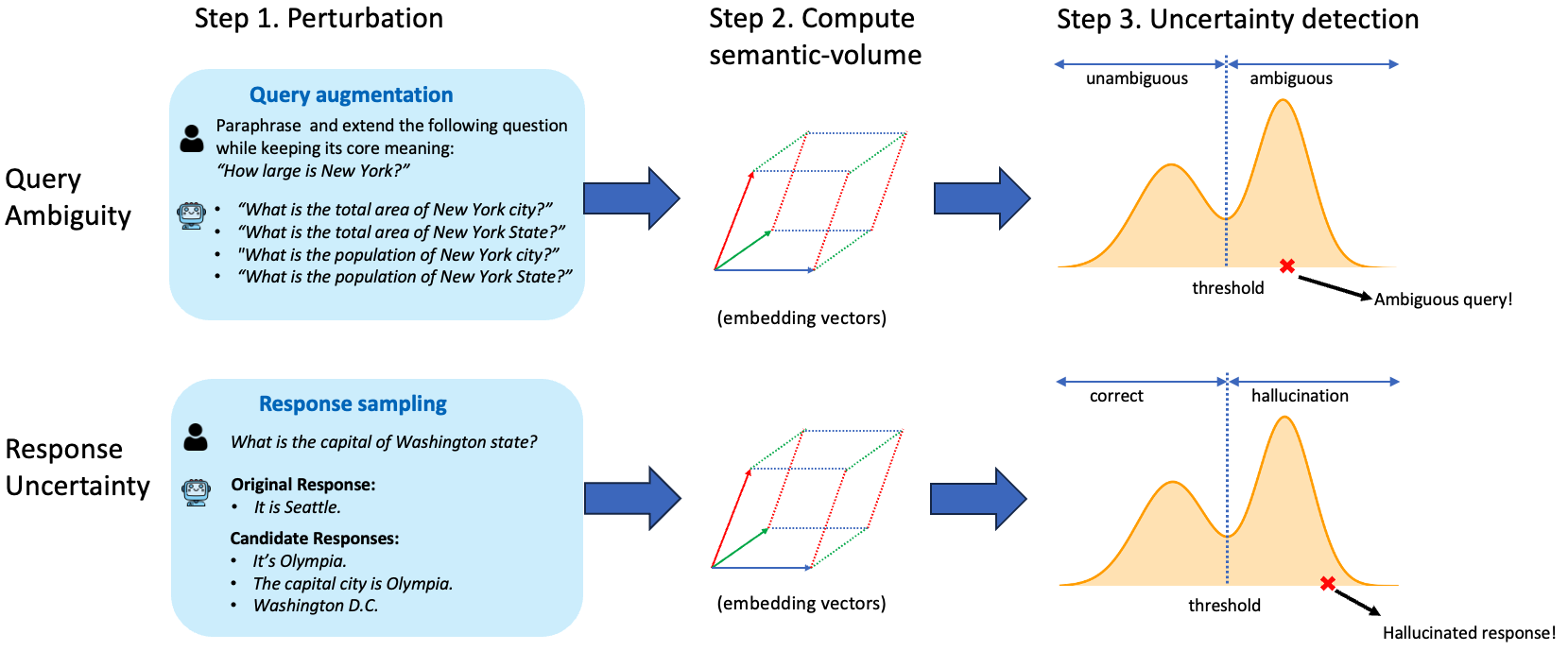}
    \caption{Pipeline for external and internal uncertainty detection using semantic volume. \textit{Step 1.} Generate perturbations. For external uncertainty, we augment each query using an LLM, treating the augmentations as perturbations. For internal uncertainty, perturbations refer to multiple sampled candidate responses.  \textit{Step 2.} Compute semantic volume (essentially $\log \det (\mV^\top \mV)$ where columns of $\mV$ are normalized embedding vectors). \textit{Step 3.} Cases with high semantic volume are predicted as ambiguous queries (external uncertainty) or hallucinated responses (internal uncertainty).}
    \label{fig:pipeline}
\end{figure*}

\section{Related Work}\label{sec:RelatedWork} 

\subsection{Hallucination} \label{subsec:RelatedWork-Hallucination}
LLM hallucinations occur when the model generates incorrect, incomplete, fabricated, or misleading outputs \citep{ji2023survey, huang2023survey, bang2023multitask, guerreiro2023hallucinations, chen2022towards}. Generally the hallucination can be caused by the lack of knowledge of the LLM itself (internal uncertainty), but could also originate from the ambiguity in the user's query (external uncertainty). Most LLMs are not explicitly trained to handle ambiguous queries and often generate incorrect responses instead, leading to hallucinations \citep{kim2024aligning}. Addressing external uncertainty requires a different approach, such as prompting the LLM to ask clarification questions before generating a response \citep{zhang2024clamber, min2020ambigqa, kuhn2022clam, lee2023asking}. In contrast, internal uncertainty caused by knowledge gaps can be mitigated through methods like retrieval-augmented generation \citep{lewis2020retrieval}, reasoning-based techniques \citep{wei2022chain, guo2025deepseek, openai2024learning}, or simply turning to stronger LLMs or human agents.

\subsection{External Uncertainty} \label{subsec:RelatedWork-External}
Query ambiguity detection is typically performed using LLMs with various prompting techniques \citep{kuhn2022clam, kim2024aligning, min2020ambigqa, zhang2024clamber, chi2024clarinet, yin2023large, lee2023asking}. For instance,  \citet{kuhn2022clam} uses LLM prompting for both detecting ambiguity and generating clarification questions. \citet{kim2024aligning} prompts LLM to disambiguate question $x$ itself, and then measure the difference between the $x_{new}$ and $x$. Larger difference above a threshold indicates ambiguity. \citet{min2020ambigqa} introduced the AmbigQA dataset, which contains ambiguous queries and their answers. Following this, \citet{zhang2024clamber} proposed an ambiguity taxonomy and introduced the CLAMBER benchmark with binary ambiguity labels. They further evaluated various models on ambiguity detection under different settings, including zero-shot vs. few-shot and chain-of-thought (CoT) prompting vs. standard prompting.

\subsection{Internal Uncertainty} \label{subsec:RelatedWork-Internal}
Many studies propose uncertainty measures for hallucination detection \citep{kuhn2023semantic, farquhar2024detecting, cole2023selectively, kadavath2022language, malinin2020uncertainty, fomicheva2020unsupervised, kossen2024semantic, lin2023generating, liu2024uncertainty, quevedo2024detecting}; similar to these, we focus on cases where LLM mistakes coincide with high uncertainty, and cases where an LLM hallucinates with high confidence are beyond the scope of this paper. There are generally two genres: \textbf{probability-based}, using information such as the token probability or entropy, and \textbf{sampling-based}, which samples more responses and measure the dispersion of the answers.

\textbf{Probability-based.}
\textit{Last Token Entropy}, which essentially uses the entropy of the vocabulary distribution at the last token, is a widely used measure of uncertainty \citep{malinin2020uncertainty}.  \textit{Log Probabilities} average log conditional token probabilities \citep{malinin2020uncertainty, quevedo2024detecting}. \citet{quevedo2024detecting} tries multiple ways to aggregate the token probabilities, such as the minimal and averaged token probabilities.

\textbf{Sampling-based.}
\citet{kuhn2023semantic} proposed \textit{Semantic Entropy} to measure uncertainty in natural language and \citet{farquhar2024detecting} applied it to detect hallucinations in large language models. Essentially for each query, they generate multiple answers and then cluster them by the same semantic meanings. Then discrete entropy calculated from the sizes of different clusters is defined as the semantic entropy.  In \citet{kadavath2022language} and \citet{cole2023selectively}, they sample multiple answers and let LLM to judge the uncertainty based on these answers. this method is called \textit{p(True)} in \citet{kuhn2023semantic}. The \textit{Lexical Similarity} method \citep{fomicheva2020unsupervised, grewal2024improving}  considers the averaged similarity of the sampled answers, and lower similarity indicates higher dispersion. 

\section{Method}

\subsection{Definitions and Notations:} \label{subsec:Definitions}
Denote $[k] \bydef \{1, 2, \dots, k\}$ for any $k \in \mathbb{N}$. For the task of external uncertainty detection, we denote the query dataset by $\calD_Q = \{q_i\}_{i\in [N_Q]}$, along with a tiny labeled subset $\calL_Q \subseteq \calD_Q$ (used to determine optimal semantic volume threshold and each query is assigned a binary label indicating whether it is ambiguous). For internal uncertainty detection, we define the query-response dataset as $\calD_R = \{(q_i, r_i)\}_{i\in [N_R]}$ with a labeled subset $\calL_R \subseteq \calD_R$. Since our method and analysis apply to both tasks, we often drop the subscript and use the general notation  $\calD = \{s_i\}_{i\in [N]}$ and $\calL \in \calD$, where $s_i$ represents a query for external uncertainty and a query-response pair for internal uncertainty.

\textbf{Volume.}
Given normalized embedding vectors $\mV = [\vv_1 \vv_2 \dots \vv_n]$ (each $\|\vv_i\|=1$), we define the squared volume as: 
\begin{equation}\label{eq:volume}
    \mathbf{Vol}^2(\mV) \bydef \det(\mV^\top \mV).
\end{equation}
The term ``volume'' originates from the fact that geometrically, $\sqrt{\det(\mV^\top \mV)}$ represents the volume of the parallelepiped spanned by vectors $\{\vv_i\}$. For example, in the three-dimensional case, where $\mV = [\vv_1\ \vv_2\ \vv_3]$, it can be verified that  $\sqrt{\det(\mV^\top \mV)}$ precisely computes $|\vv_1^\top (\vv_2 \times \vv_3)|$, which corresponds to the volume of the three-dimensional parallelepiped formed by $\{\vv_1, \vv_2, \vv_3\}$. A more detailed discussion on the geometric interpretation of this measure is provided in Appendix~\ref{sec:Appendix-GeometryOfVolume}.

\textbf{Semantic Volume.}
To avoid numerical singularities from duplicate embeddings (e.g., two sampled responses or extended queries are identical), we add a small perturbation $\epsilon\mI$ with $\epsilon = 10^{-10}$ and compute $\det(\mV^\top \mV+  \epsilon\mI)$ to maintain numerical stability (in Appendix~\ref{sec:Appendix-Epsilon}, we show that $\epsilon$ is negligible to the spectral norm  $\|\mV^\top \mV\|$ and hence it only serves to ensure numerical stability and does not affect the quantification). In practice, the absolute values of squared volumes are often small due to the nature of our perturbations. Therefore we take the logarithm, leading to the formulation
\begin{equation}\label{eq:log-vol}
    \log \mathbf{Vol}^2(\mV) \bydef \log \det\left(\mV^\top \mV + \epsilon\, \mI_n\right)
\end{equation}
Moreover, we apply Principal Component Analysis (PCA) to reduce dimensionality by projecting the vectors onto the top $d$ principal components, obtaining $\tilde{\mV} = [\tilde{\vv}_1 \tilde{\vv}_2 \dots \tilde{\vv}_n]  \in \R^{d \times n}$, where each $\tilde{\vv}_i \bydef \mP_{PCA}\vv_i$. Here, $\mP_{PCA}$ is the projection matrix. This results in the general form of our final semantic volume measure:
\begin{equation}
\label{eq:semantic-vol}
\begin{aligned}
    \text{SemanticVolume}(\mV) 
    &\bydef \log \mathbf{Vol}^2(\mP_{PCA} \mV) \\
    &= \log \det\left(\tilde{\mV}^\top \tilde{\mV} + \epsilon\, \mI_n\right)
\end{aligned}
\end{equation}

\subsection{Semantic Volume Uncertainty Detection Algorithm:}
Our algorithm using semantic volume to detect high uncertainty is outlined below (the overall pipeline is illustrated in Figure~\ref{fig:pipeline} and the detailed pseudocode is provided in Algorithm~\ref{algo}). Additionally, we present a  Case Study in Appendix~\ref{sec:Appendix-CaseStudy}, analyzing representative examples of queries and responses exhibiting high versus low uncertainty.

\noindent\textbf{Step 1: Augmentation.}
For each $s \in \calD$, we augment it with $n$ perturbations. Precisely, for external uncertainty, we prompt an LLM to augment/paraphrase each query $s$ to obtain $n$ perturbed versions, while for internal uncertainty, we sample $n$ candidate responses.

\noindent\textbf{Step 2: Compute semantic volume.}
We obtain embedding vectors using Sentence-Transformer \citep{reimers2019sentence}, normalize them, and apply PCA dimension reduction, yielding $\tilde{\mV} = [\tilde{\vv}_1 \tilde{\vv}_2 \dots \tilde{\vv}_n]  \in \R^{d \times n}$ for $n$ perturbations. Then compute the semantic volume according to \eqref{eq:semantic-vol}.

\noindent\textbf{Step 3: Uncertainty detection.}
A higher semantic volume indicates greater uncertainty. To determine the optimal threshold $\tau^*$, we use a tiny labeled random subset $\calL \subseteq \calD$ with size 100 for threshold tuning \footnote{Note that this labeled subset is only used for finding a more precise threshold. In practice, one can consider a completely unsupervised setting and use a simpler heuristic threshold such as the median.} (the exact formula for $\tau^*$ is characterized in Proposition~\ref{prop:OptimalTauFormula}). Finally, we classify the entire dataset $\calD$ by assigning binary uncertainty labels based on the semantic volume threshold.

\section{Experiment: External Uncertainty}\label{sec:Experiment-External}

\subsection{Experimental Setup}\label{subsec:ExperimentSetup-External}
\noindent \textbf{Data.}
We use two benchmark datasets. The first is CLAMBER \citep{zhang2024clamber}, a balanced dataset of 3K queries, each annotated with a binary label indicating whether it is ambiguous. The second is a balanced subset of 5K samples from the AmbigQA dataset \citep{min2020ambigqa}. Each question in AmbigQA is labeled as either unambiguous or ambiguous; ambiguous questions are further annotated with multiple disambiguated versions and corresponding answers, each reflecting a distinct plausible interpretation.

\noindent \textbf{Models.}
For query augmentation, we use \texttt{Claude3.5-Sonnet} \citep{anthropic_claude3_2023} and the prompt is provided in Appendix~\ref{sec:Appendix-Prompts}. \texttt{Qwen2-1.5B-instruct} \citep{yang2024qwen2technicalreport} is used as the sentence-transformer to generate embeddings.

\noindent \textbf{Evaluation.}
We conduct experiments for binary classification tasks on ambiguity of the queries. The performance is assessed by comparing the predicted binary labels against the ground truth labels, reporting both accuracy and F1 score.

\noindent \textbf{Baselines.}
Note that the sampling-based methods discussed in Section~\ref{subsec:RelatedWork-Internal} can be naturally extended to query ambiguity detection if we can generate analogous perturbations of the queries, similar to how candidate responses are sampled in response uncertainty detection. Below, we outline the baseline methods we consider. Some of these were originally designed for response uncertainty, but we readily adapt their methodology to query ambiguity.
\begin{itemize}
    \item \textit{Type 1: Prompting-based.} 
    Directly prompt LLMs to determine whether a given query is ambiguous. We evaluate the following models: Vicuna-13B \citep{vicuna2023blog}, Llama2-13B-Instruct \citep{touvron2023llama2}, Llama2-70B-Instruct \citep{touvron2023llama2}, Llama3.2-3B-Instruct \citep{MetaAI2024Llama3.2}, and ChatGPT \citep{achiam2023gpt}. 
    We also include ChatGPT results from CLAMBER using few-shot and chain-of-thought prompting \citep{zhang2024clamber}.

    \item \textit{Type 2: Probability-based}. 
    Using the token probabilities to quantify uncertainty. These methods require access to the model’s internal token probabilities. We consider the follow methods and use \texttt{Llama3.2-1B-Instruct} \citep{MetaAI2024Llama3.2} to obtain the token probabilities. (a) \textbf{Last Token Entropy} \citep{kadavath2022language, arora2021types, malinin2020uncertainty}: computes the entropy of the vocabulary distributions at the last token of the query. (b) \textbf{Log Probabilities} \citep{malinin2020uncertainty, quevedo2024detecting}: measures uncertainty by computing the log of the product of conditional token probabilities, which is equivalent to summing the log conditional probabilities across all tokens in the query.

    \item \textit{Type 3: Sampling-based}. These methods originally measure the variation of sampled responses to quantify response uncertainty. Here for query ambiguity, we adapt and apply them to perturbed variations of queries.
    (a) \textbf{p(True)} \citep{kadavath2022language, cole2023selectively}: the original p(True) method quantifies uncertainty based on the probability of the LLM's output. We adopt it to directly use the LLM's answer and compare it against the ground truth binary labels.
    (b) \textbf{Lexical Similarity} \citep{lin2023generating, fomicheva2020unsupervised}: computes the averaged pairwise similarity of perturbed queries.  
    (c) \textbf{Semantic Entropy} \citep{kuhn2023semantic}: clusters the perturbations into semantic equivalence classes and computes the entropy over the clusters.
\end{itemize}

\subsection{Results}
The performance of our method and baseline approaches on the query ambiguity classification task is presented in Table~\ref{tab:Results-External}. Here we choose $n=20$ for the augmentations for queries (the discussion on varying the perturbation size $n$ is provided in Appendix~\ref{sec:Appendix-Variation-n}). The original CLAMBER dataset includes a diverse range of ambiguities, such as queries involving unfamiliar entities, self-contradictions, multiple meanings, and missing context. We observe that identifying ambiguous queries remains challenging for LLMs, even for powerful models like ChatGPT, despite various prompting strategies (few-shot and CoT). Among the baselines, probability-based methods achieve higher F1 scores but a critical limitation of them is that they require access to token probabilities,  which is not provided for most of the close-sourced models.  Among sampling-based methods, semantic entropy generally performs better. Nonetheless, our semantic volume method significantly outperforms all three categories of baselines, demonstrating its effectiveness in detecting ambiguous queries. For AmbigQA task, results indicate that ambiguity detection on this dataset is generally more challenging compared to CLAMBER. Nonetheless, our Semantic Volume method continues to outperform the baseline methods.

\begin{table*}[t!]
\centering
\renewcommand{\arraystretch}{1.09}   
\resizebox{0.61\textwidth}{!}{%
\begin{tabular}{l|cc|cc}
\hline
\multirow{2}{*}{\textbf{Method}} &
\multicolumn{2}{c|}{\textbf{CLAMBER}} &
\multicolumn{2}{c}{\textbf{AmbigQA}} \\ \cline{2-5}
& \textbf{Acc.} & \textbf{F1} & \textbf{Acc.} & \textbf{F1} \\ \hline
Vicuna-13B (zero-shot)                 & 50.6 & 39.9 & 45.1 & 19.5 \\
Llama2-13B-Instruct (zero-shot)        & 45.6 & 43.6 & 46.2 & 22.4 \\
Llama2-70B-Instruct (zero-shot)        & 50.3 & 34.2 & 49.4 & \underline{64.5} \\
Llama3.2-3B-Instruct (zero-shot)       & 51.5 & 37.7 & 49.3 & 33.8 \\
ChatGPT (zero-shot)                    & 54.3 & 53.4 & 54.8 & 52.8 \\
ChatGPT (few-shot)                     & 51.6 & 49.2 & \underline{54.9} & 53.1 \\
ChatGPT (zero-shot + CoT)              & \underline{57.3} & 56.9 & 54.2 & 55.1 \\
ChatGPT (few-shot + CoT)               & 53.6 & 51.4 & 54.0 & 50.8 \\ \hline
Last Token Entropy                     & 52.2 & \underline{67.3} & 48.1 & 63.2 \\
Log Probabilities                      & 45.5 & 66.0 & 50.3 & 62.8 \\ \hline
pTrue (Llama3-8B-Instruct)             & 52.6$_{1.19}$ & 53.2$_{1.44}$ & 50.1$_{0.88}$ & 28.1$_{2.60}$ \\
pTrue (Mistral-7B-Instruct)            & 47.2$_{1.16}$ & 26.3$_{2.28}$ & 49.8$_{0.35}$ & 24.3$_{3.02}$ \\
Lexical Similarity                     & 52.8$_{0.41}$ & 53.7$_{0.75}$ & 48.9$_{1.07}$ & 24.9$_{2.28}$ \\
Semantic Entropy                       & 50.2$_{0.33}$ & 62.8$_{0.56}$ & 51.4$_{0.97}$ & 61.3$_{2.59}$ \\ \hline
\textbf{Semantic Volume (ours)}        & \textbf{58.0$_{0.18}$} & \textbf{69.1$_{0.32}$} & \textbf{55.7$_{0.95}$} & \textbf{67.6$_{0.83}$} \\ \hline
\end{tabular}}
\caption{External uncertainty: Accuracy and F1 on CLAMBER and AmbigQA. Sampling‐based methods report mean $\pm$ std.\ over three trials (std.\ as subscript).}
\label{tab:Results-External}
\end{table*}

\section{Experiment: Internal Uncertainty}\label{sec:Experiment-Internal}

\subsection{Experimental Setup}

\noindent \textbf{Data.} We use subsets of the TriviaQA \citep{joshi2017triviaqa} and SQuAD \cite{rajpurkar2016squad} datasets, which are reading comprehension benchmarks containing questions paired with reference answers. For each task, we generate responses using the evaluation LLMs (see model details below) with a temperature of 0. A response $y$ is flagged as a hallucination if the ROUGE-L score with respect to the reference answer $y_{ref}$ falls below 0.3 (i.e. the label is defined as $\mathbf{1}_{RougeL(y, y_{ref}) < 0.3}$), following the same metric used in \citet{kossen2024semantic} \footnote{We also conducted a human validation on 1,000 randomly sampled examples and found a 95.1\% agreement with this labeling criterion.}. We retain 2500 data labeled as hallucinations and 2500 labeled as correct, constructing a balanced 5K dataset for each task. 

\noindent \textbf{Models.} 
We test on these models:  \texttt{Llama3.2-\allowbreak 1B-\allowbreak Instruct}, \texttt{Llama3-8B-Instruct}, \texttt{Qwen2.5-\allowbreak 1.5B-\allowbreak Instruct}, \texttt{Qwen3-\allowbreak 8B}, \texttt{Qwen3-\allowbreak 14B} and \texttt{Mistral-\allowbreak 7B-\allowbreak Instruct}.
For sampling-based methods, we generate candidate responses with temperature 1. For embeddings, we use the same sentence-transformer as in Section~\ref{sec:Experiment-External} (ablation study of different sentence-transformers can be found in Appendix~\ref{sec:Appendix-Variation-EmbeddingModels}).

\noindent \textbf{Evaluation.} 
We compare the predicted binary hallucination labels against the ground truth labels and report both accuracy and F1 score. Furthermore, we also add the evaluation based on the AUROC (area under the receiver operator characteristic curve) metric, which compares the raw uncertainty scores against the ground truth labels.  In fact, for a given uncertainty measure $m(\cdot)$, the AUROC score is equivalent to  $\P[m(y_{hallucinated}) > m(y_{correct})]$, where $y_{hallucinated}$ and $y_{correct}$ are randomly chosen hallucinated and correct answers, respectively. Hence a higher AUROC (closer to 1) indicates that the uncertainty measure more effectively distinguishes hallucinated responses by assigning them higher uncertainty scores. AUROC is a widely used metric in many existing studies on response hallucination detection \citep{kuhn2023semantic, kossen2024semantic, kadavath2022language}.

\noindent \textbf{Baselines.}
We adapt the baseline methods from Section~\ref{subsec:ExperimentSetup-External}, applying them to sampled responses instead of augmented queries.

\subsection{Results}
The performance results for \texttt{Llama3.2-\allowbreak 1B-\allowbreak Instruct}  under both TriviaQA and SQuAD tasks are presented in Tables~\ref{tab:Results-Internal} (results for other models can be found in Appendix~\ref{sec:Appendix-Variation-LargeResponseGenerationModel}). For sampling-based methods, we set the response sampling size to $n=20$.  Notably, our semantic volume method significantly outperforms all baselines in both accuracy and F1 score. Furthermore, the AUROC results confirm that our semantic volume serves as a highly effective uncertainty signal for hallucination detection. Additionally, we observe that when comparing p(True), which includes sampled candidate responses as context, to direct prompting, the inclusion of context degrades the performance. Moreover, for methods that rely on prompting LLMs (including pTrue), we find that LLMs sometimes exhibit a strong bias toward answering almost all ``Yes'' or all ``No''. In fact in Table~\ref{tab:Results-Internal}, both \textit{Prompt Llama3.2-1B-Instruct} and \textit{pTrue (Llama3.2-1B-Instruct)} exhibit this behavior, nearly predicting all responses as hallucinations (further discussions are provided in Appendix~\ref{sec:Appendix-Variation-LargeResponseGenerationModel}). This instability highlights another drawback of such methods that rely on LLM prompting for uncertainty estimation. From both tables, we observe that sampling-based methods that measure the dispersion of sampled responses (particularly lexical similarity and semantic entropy) generally outperform probability-based methods, which aligns with the findings in \citet{cole2023selectively}.

\begin{table*}[t!]
\centering
\renewcommand{\arraystretch}{1.28}   
\resizebox{0.78\linewidth}{!}{%
\begin{tabular}{l|ccc|ccc}
\hline
\multirow{2}{*}{\textbf{Method}} &
\multicolumn{3}{c|}{\textbf{TriviaQA}} &
\multicolumn{3}{c}{\textbf{SQuAD}} \\ \cline{2-7}
& \textbf{Acc.} & \textbf{F1} & \textbf{AUROC} & \textbf{Acc.} & \textbf{F1} & \textbf{AUROC} \\ \hline
Prompt Llama3.2-1B-Instruct          & 50.5 & 66.8 & N/A & 50.6 & 22.7 & N/A \\
Prompt Llama3-8B-Instruct            & 65.5 & 60.1 & N/A & \underline{64.0} & 65.5 & N/A \\
Prompt Mistral-7B-Instruct           & \underline{68.7} & 61.8 & N/A & 60.7 & 57.3 & N/A \\ \hline
Last Token Entropy                   & 60.1 & 59.9 & 63.9 & 53.6 & 34.2 & 56.6 \\
Log Probabilities                    & 60.1 & 62.9 & 65.5 & 54.8 & 46.4 & 56.7 \\ \hline
pTrue (Llama3.2-1B-Instruct)         & 49.5$_{0.37}$ & 64.4$_{0.81}$ & 61.2$_{0.91}$ & 50.1$_{0.81}$ & 9.5$_{3.26}$ & 52.3$_{1.11}$ \\
pTrue (Llama3-8B-Instruct)           & 63.7$_{2.04}$ & 45.4$_{4.23}$ & 56.9$_{3.25}$ & 58.2$_{0.96}$ & 45.4$_{1.63}$ & 53.5$_{1.30}$ \\
pTrue (Mistral-7B-Instruct)          & 62.8$_{0.86}$ & 46.3$_{1.12}$ & 65.4$_{1.26}$ & 51.5$_{1.22}$ & 11.3$_{3.65}$ & 51.7$_{0.89}$ \\
Lexical Similarity                   & 64.6$_{0.61}$ & \underline{72.2$_{0.36}$} & 73.3$_{0.39}$ & 58.3$_{0.54}$ & 59.9$_{0.66}$ & 62.4$_{0.71}$ \\
Semantic Entropy                     & 63.8$_{0.98}$ & 69.7$_{1.08}$ & \underline{73.9$_{0.75}$} & 61.8$_{1.51}$ & \underline{68.8$_{1.75}$} & \underline{64.9$_{1.17}$} \\ \hline
\textbf{Semantic Volume (ours)}      & \textbf{72.4$_{0.55}$} & \textbf{75.5$_{0.34}$} & \textbf{79.7$_{0.16}$} & \textbf{64.7$_{0.49}$} & \textbf{71.2$_{0.63}$} & \textbf{68.8$_{0.56}$} \\ \hline
\end{tabular}}
\caption{Internal uncertainty: Accuracy, F1, and AUROC on TriviaQA and SQuAD.  “N/A’’ indicates AUROC is not applicable (no probabilistic scores).  For sampling-based methods, we report mean $\pm$ std.\ over three trials (std.\ shown as subscript).}
\label{tab:Results-Internal}
\end{table*}

\subsection{Distribution Separation}
In this section, we compare the distributions of various uncertainty measures using visualization and the Kolmogorov–Smirnov (KS) test \citep{smirnov1948table}. Specifically, we plot histograms for the hallucinated subset (label 1) and the correct subset (label 0) from our TriviaQA dataset. Ideally, a well-performing measure should yield two distinct bulks, with greater separation indicating stronger discriminative power. To quantitatively assess the separation, we perform a two-sample Kolmogorov–Smirnov test, a non-parametric test that compares two empirical distributions by measuring the maximum distance between their empirical cumulative distribution functions. A large KS statistic combined with a small $p$-value suggests that the two distributions are significantly different.

The plots and statistics are shown in Figure~\ref{fig:distribution_separation}. Combined the KS statistics and the histograms, we observe that indeed \textit{Last Token Entropy} and \textit{Log Probabilities} struggle to effectively separate the two groups of data, while \textit{Lexical Similarity}, \textit{Semantic Entropy}, and our \textit{Semantic Volume} exhibit stronger separation. Particularly, we note that the distribution of semantic entropy closely resembles that of our semantic volume measure. In fact, we will provide theoretical analysis in Section~\ref{sec:Theory} showing that our semantic volume can be interpreted as the differential entropy of the semantic embedding vectors, and can be viewed as a more general and continuous version of semantic entropy.

\begin{figure}[ht]
    \centering
    \begin{subfigure}{0.30\linewidth}
        \centering
        \includegraphics[width=\linewidth]{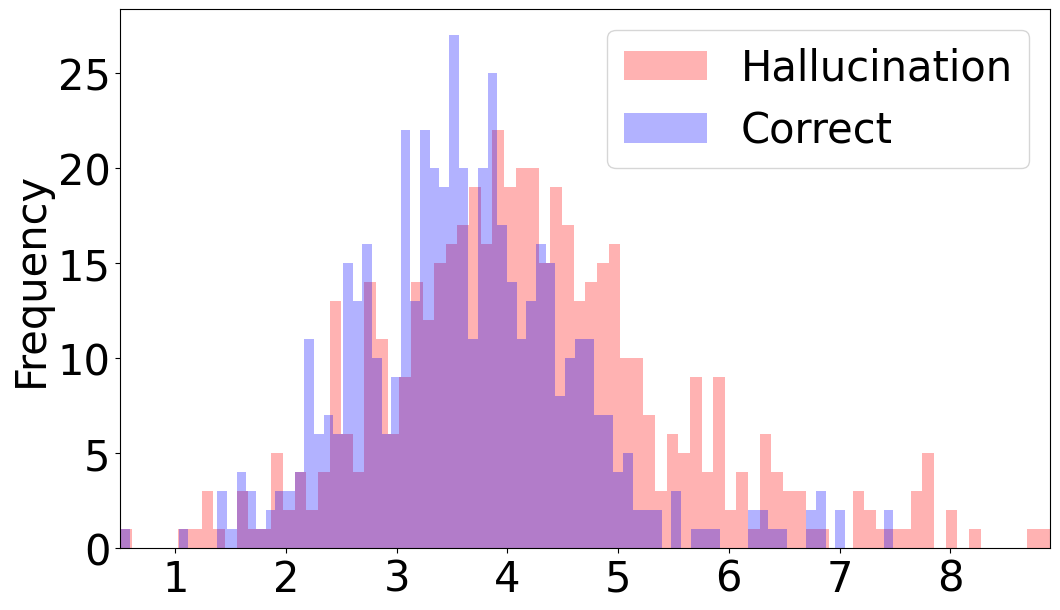}
        \caption{Last Token Entropy (KS Statistic 0.238 with $p$-value 1e-36).}
    \end{subfigure} \hfill
    \begin{subfigure}{0.30\linewidth}
        \centering
        \includegraphics[width=\linewidth]{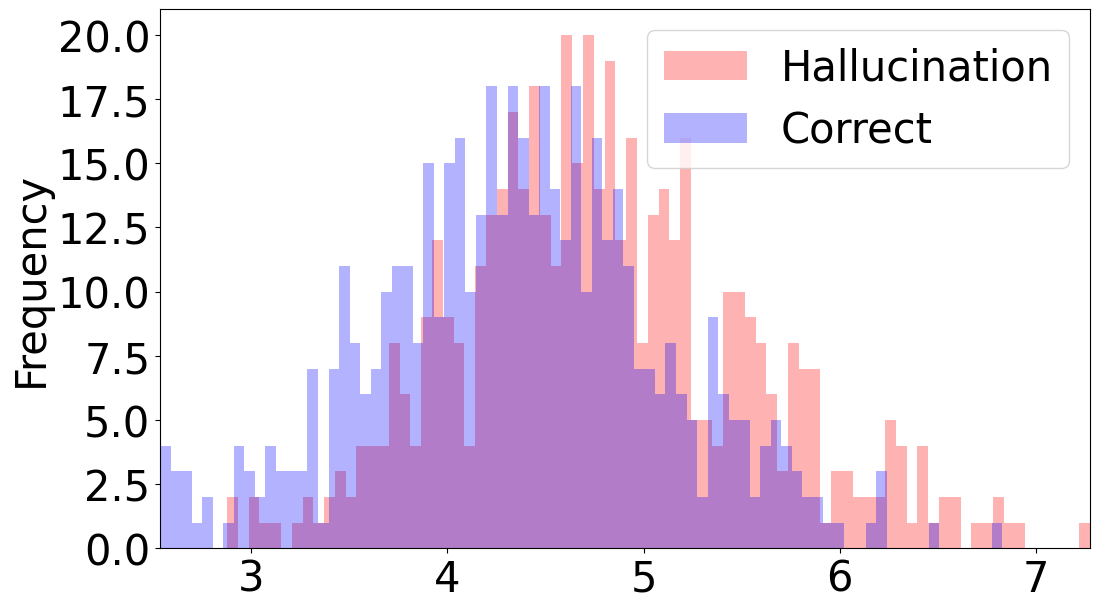}
        \caption{Log Probabilities (KS Statistic 0.226 with $p$-value 5e-27).}
    \end{subfigure} \hfill
    \begin{subfigure}{0.30\linewidth}
        \centering
        \includegraphics[width=\linewidth]{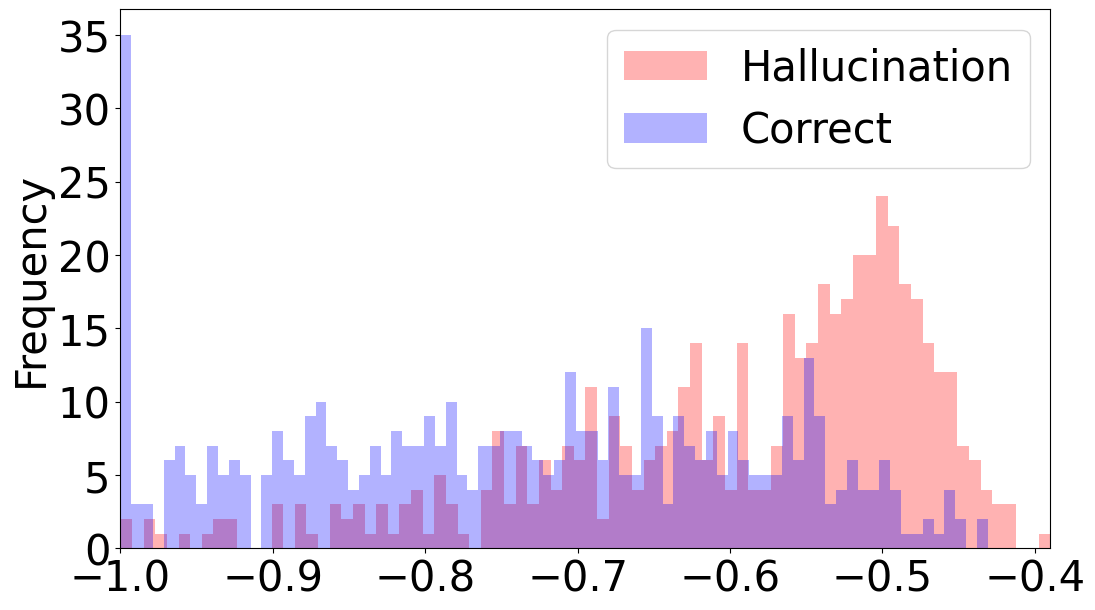}
        \caption{Lexical Similarity (KS Statistic 0.428 with $p$-value 1e-26).}
    \end{subfigure} \\[1ex]
    \begin{subfigure}{0.30\linewidth}
        \centering
        \includegraphics[width=\linewidth]{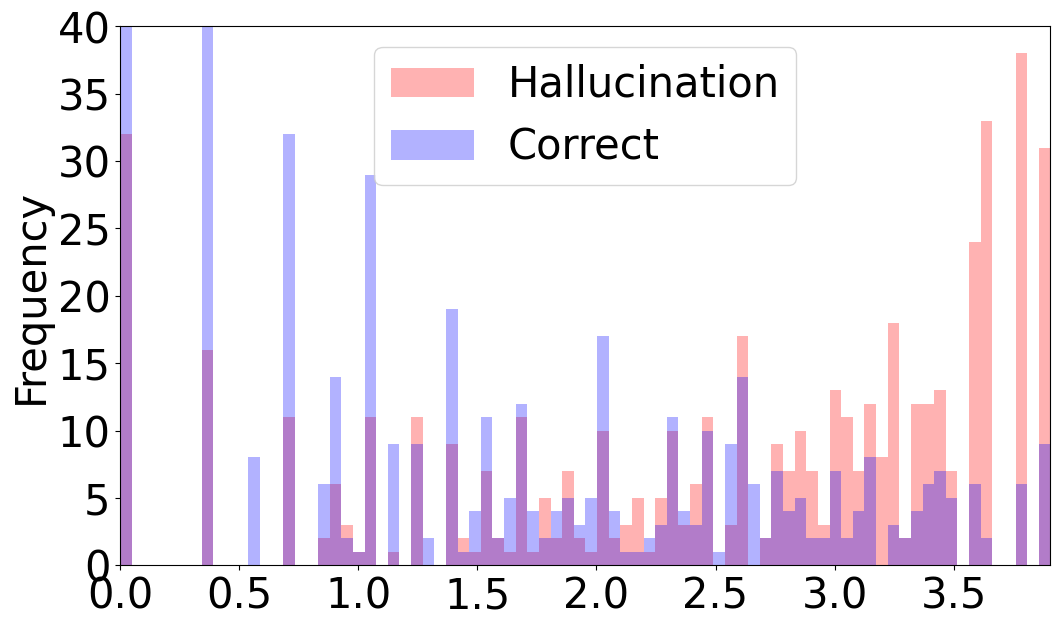}
        \caption{Semantic Entropy (KS Statistic: 0.376 with $p$-value 7e-32).}
    \end{subfigure} 
     \hspace{8mm}
    \begin{subfigure}{0.30\linewidth}
        \centering
        \includegraphics[width=\linewidth]{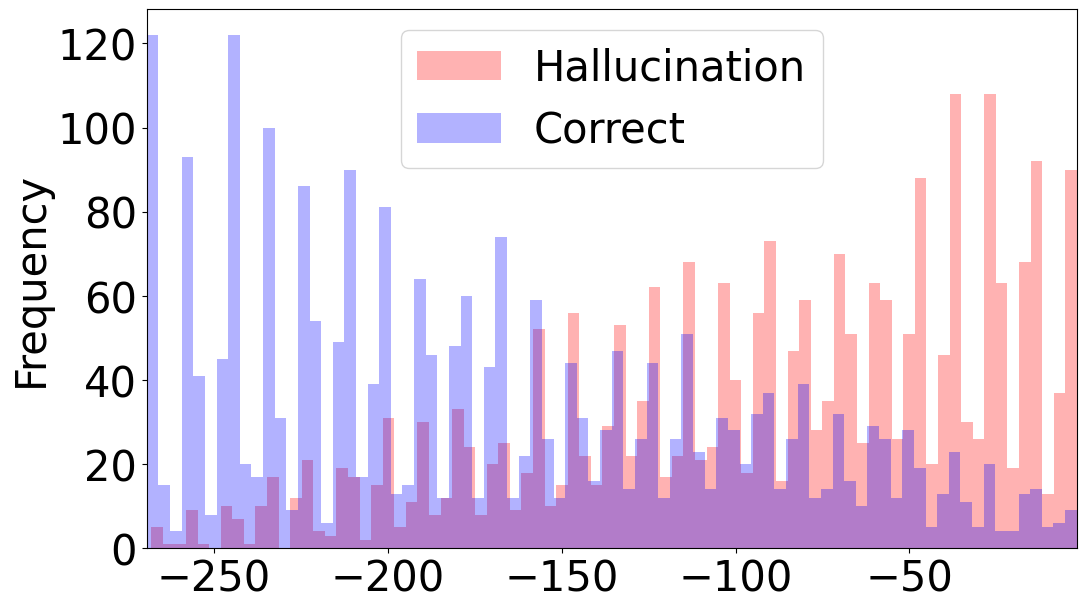}
        \caption{Semantic Volume (KS Statistic 0.446 with $p$-value 3e-54).}
    \end{subfigure}
    \caption{Distribution of subsets for both labels across different uncertainty measures. Two-sample Kolmogorov–Smirnov statistics are computed to quantify the separation of two bulks.}
    \label{fig:distribution_separation}
\end{figure}

\section{Ablation Study and Hyperparameter Analysis}
The original embedding dimension from the sentence-transformer is 1536. In the external uncertainty experiment, we reduce the dimensionality using PCA with $d = 10$, while in the internal uncertainty experiment, we set $d = 20$. Ablation studies confirm PCA-projected vectors $\tilde{\mV}$ outperform raw embeddings (see Appendix~\ref{sec:Appendix-Variation-d}). This suggests that lower-dimensional projections help separate perturbation vectors more effectively. Furthermore, we explore various values of $d$ to analyze the effect of dimensionality on performance. Our findings indicate that the optimal dimension $d$ is task-dependent, with diminishing improvements beyond a certain point. Nonetheless, it is important to note that even without PCA dimension reduction, our method still outperforms various baselines in both external and internal uncertainty detection tasks (see Figure~\ref{fig:Ablation}).

To demonstrate the generalizability of our method and study the effect of various hyperparameters and model choices. In Appendix~\ref{sec:Appendix-Variation-n}, we analyze the effect of varying $n$, the number of perturbations. As expected, larger $n$ yields more accurate uncertainty estimation but increases computational cost, while smaller $n$ reduces cost but may sacrifice accuracy. We choose $n=20$ to balance performance and efficiency. In Appendices~\ref{sec:Appendix-Variation-EmbeddingModels} and \ref{sec:Appendix-Variation-LargeResponseGenerationModel}, we examine the impact of different embedding models and response generation models. We find that larger embedding models provide little additional performance gain. Furthermore, when detecting hallucinations in responses generated by a larger LLM, the performance of most methods slightly declines. However, our Semantic Volume method still outperforms all baselines.

\section{Theoretical Analysis}\label{sec:Theory}
In this section, we present theoretical analyses, with proofs and additional supporting lemmas provided in Appendix~\ref{sec:Appendix-Proofs}. First, we derive the exact formula for the optimal threshold $\tau^*$ in Proposition~\ref{prop:OptimalTauFormula}.
Then in Theorems~\ref{thm:DiffEntropy} and \ref{thm:Equivalence}, we show that under Gaussian distribution assumptions of the perturbations, our semantic volume measure effectively computes the differential entropy of the embedding vectors, using entropy as a measure of uncertainty detection. This insight allows us to naturally view our method as a generalization of semantic entropy \citep{kuhn2023semantic}. Notably, semantic entropy involves a manual clustering step, and only considers the entropy between clusters while ignoring discrepancies within clusters (since the responses in the same cluster are semantically similar but not exactly identical). In contrast, \textbf{our method more generally captures the overall semantic dispersion across all sampled perturbations}, providing a more comprehensive uncertainty measure.


\begin{proposition}[Formula for optimal threshold $\tau^*$]\label{prop:OptimalTauFormula}
Denote $\calL \in \calD$  as a labeled subset with inputs $\{s_i\}$ and labels $\{y_i\}$:
\[
\mathcal{L} = \{(s_i, y_i)\}_{i=1}^M \quad \text{with} \quad y_i \in \{0, 1\},
\]
For each $s_i$, denotes its semantic volume as $m(s_i)$. Define a classification rule
\[
\hat{y}_i(\tau) =
\begin{cases}
1, & m(s_i) > \tau, \\
0, & \text{otherwise}.
\end{cases}
\]
Then the optimal threshold $\tau$ that maximizes the $F_1$ score on $\mathcal{L}$ is given by
\begin{equation}\label{eq:optimal_tau}
    \tau^* \bydef \arg \max_{\tau \in \mathbb{R}} F_1(\tau) = \arg \max_{\tau \in \mathbb{R}} \left( \frac{2 \sum_{i=1}^M \mathbf{1}_{\hat{y}_i(\tau) = y_i = 1} }{\sum_{i=1}^M \mathbf{1}_{\hat{y}_i(\tau) = 1} + \mathbf{1}_{y_i=1} }  \right).
\end{equation}
Note that the $F_1$ score can be replaced by other metrics, such as the accuracy. 


\end{proposition}

\begin{theorem}\label{thm:DiffEntropy}
Denote the embedding vector and normalized embedding vector of the original text (either a query or a response) as $\bvx$ and $\bvv$, respectively. Denote the perturbation embeddings as $\mX \bydef  [\vx_1 \vx_2 \dots \vx_n] \in \R^{d \times n}$ and the normalized perturbation embeddings as $\mV \bydef  [\vv_1 \vv_2 \dots \vv_n] \in \R^{d \times n}$ (i.e. $\vx_i = \vv_i /  \|\vv_i\|$ for each $i \in [n]$). Assume Gaussian distribution $\vx_i \sim \calN(\bvx, \mSigma)$. Then in high-dimensional regime, $\log \det(\mV^\top \mV)$ corresponds to the shifted differential entropy of the perturbations $\{\vx_i\}_{k\in[n]}$. That is,
\[
\log \det(\mV^\top \mV)\ \dot{=}\ \calH(\mX) + C,
\]
\end{theorem}
where $\calH(\mX) \bydef -\E_{x\sim \mX} \left[ \log p_{\mX}(x) \right] $ is the differential entropy and $C$ is a constant offset term.

Then we obtain the following Theorem, which is a direct consequence of Theorem~\ref{thm:DiffEntropy} and Lemma~\ref{lem:LinearTransformationInvariance} in Appendix~\ref{sec:Appendix-Proofs}.

\begin{theorem} \label{thm:Equivalence}
Under the same setting and notations of the Theorem~\ref{thm:DiffEntropy}, our \textbf{Semantic Volume} method essentially generates same binary decisions compared to using \textbf{differential entropy} of the perturbation embedding vectors. That is, denote the Semantic Volume measure and differential entropy measure as $m(\cdot)$ and $\tilde{m}(\cdot)$ respectively. For the labels
\[
y_i \bydef \mathbf{1}_{m(s_i) < \tau^\ast} \quad \text{ and }\quad  \tilde{y}_i \bydef  \mathbf{1}_{\tilde{m}(s_i) < \tilde{\tau}^\ast},
\]
where $\tilde{\tau}^\ast$ is the optimal threshold for the differential entropy measure, we have
\[
\tilde{y}_i = y_i \quad \text{for all } s_i \in \calD \setminus \calL.
\]

\end{theorem}

\section{Conclusion}\label{sec:conclusion}
One limitation of our current study is that our work considers the same scope as the references in Section~\ref{subsec:RelatedWork-Internal}: we focus on the situation where uncertainty aligns with incorrectness, without addressing \textit{confidently wrong} responses, meaning an LLM hallucinates with high confidence (low uncertainty). We believe that addressing such cases requires different strategies, such as factuality checking or incorporating external knowledge for verification. We leave these explorations for future work.

In summary, we have introduced \textit{Semantic Volume}, a novel and general-purpose measure for detecting both \textit{external uncertainty} (query ambiguity) and \textit{internal uncertainty} (response uncertainty) in large language models. By generating perturbations, embedding these perturbations as normalized vectors, and computing the determinant of their Gram matrix, we obtain a measure that captures the overall semantic dispersion. Extensive experiments on benchmark datasets showed that semantic volume significantly outperforms various types of existing baselines (prompting-based, probability-based, and sampling-based) for both ambiguous query classification and response hallucination detection. Furthermore, from a theoretical standpoint, we established that semantic volume can be viewed as the differential entropy of the embedding vectors, thereby unifying and extending prior sampling-based metrics (e.g., semantic entropy). This interpretation highlights why our measure is robust and comprehensive: unlike purely clustering-based approaches, we account for the \textit{overall} dispersions in the embedding space. Moreover, our method is applicable even when the LLM is only accessible via an external API or when internal model details are unavailable, making it broadly practical across closed-source or API-based models. Overall, our findings suggest that semantic volume is a promising step toward more reliable, interpretable uncertainty detection for both external and internal uncertainty of LLMs.

\bibliography{aaai2026}



\setcounter{secnumdepth}{2}
\clearpage
\appendix






\section{Semantic-Volume Uncertainty Detection Algorithm}\label{sec:Appendix-Algo}

Here we provide the pseudocode of the Semantic-Volume uncertainty detection algorithm.
\begin{algorithm}[H]
\caption{Uncertainty Detection via Semantic-Volume}
\begin{algorithmic}[1]
\Require \\
        \begin{itemize}
            \item Dataset $\calD = \{s_i\}_{i=1}^N$ of queries or query-response pairs.
            \item A small labeled subset $\mathcal{L} \subseteq \mathcal{D}$ with binary labels.
            \item SentenceTransformer model $\calM$.
        \end{itemize} 
    \For{each string $s \in \mathcal{D}$}
        \State \textbf{Extend} $s$ via query extension/response sampling to obtain: $[s_1, s_2, \dots, s_N]$
        \State Normalize embedding vectors:
        $\mV = [\vv_1  \vv_2  \dots  \vv_N] \bydef \left[\frac{\calM(s_1)}{\|\calM{(s_1)}\|}, \frac{\calM(s_2)}{\|\calM{(s_2)}\|}, \dots, \frac{\calM(s_N)}{\|\calM{(s_N)}\|}\right ]$
        \State Apply PCA dimension reduction: $\tilde{\mV} \bydef \mP_{PCA} \mV $.
        \State Compute Semantic Volume:
        \[
        m(s) \gets \log \mathbf{Vol}(\mP_{PCA} \mV) \bydef \log \det\left(\tilde{\mV}^\top \tilde{\mV} + \epsilon\, \mI_n\right)
        \]
    \EndFor
    \State \textbf{Threshold Tuning:} Using the labeled set $\mathcal{L}$, find the threshold $\tau^*$ that maximizes the F1 score:
    \[
    \tau^* \gets \arg \max_{\tau \in \mathbb{R}} F_1(\tau) 
    \bydef \arg \max_{\tau \in \mathbb{R}} \left( \frac{2 \sum_{i=1}^{|\calL|} \mathbf{1}_{\hat{y}_i(\tau) = y_i = 1} }{\sum_{i=1}^{|\calL|} \mathbf{1}_{\hat{y}_i(\tau) = 1} + \mathbf{1}_{y_i=1} }  \right).
    \]
    \For{each $(s, m(s))$} 
        \State \textbf{Predict} $s$ with:
        \[
        \hat{y}_s\gets
        \begin{cases}
            1, & \text{if } m(s) > \tau^*, \\
            0, & \text{otherwise.}
        \end{cases}
        \]
    \EndFor
    \State \Return The fully labeled dataset: $\{(s, \hat{y}_{s}) \,:\, s \in \mathcal{D}\}$.
\end{algorithmic}
\label{algo}
\end{algorithm}

\textbf{Time Complexity}

Most of the computational cost in our algorithm comes from perturbation sampling, similar to other sampling-based methods. Using the standard LLM implementation to generate multiple candidate sequences, the cost is equivalent to one LLM inference per data point. Embedding generation is highly efficient, taking approximately 5 minutes for our CLAMBER-3K dataset and 3 minutes for the TriviaQA-5K dataset on a single NVIDIA H100-80GB GPU. The computation of Semantic Volume, which involves matrix multiplication and determinant calculation for each set of perturbations, is computationally negligible (typically within seconds for the entire dataset), compared to the sampling step. This results in an overall complexity of $O(N)$ for our algorithm. 

\section{Geometric interpretation as volume} \label{sec:Appendix-GeometryOfVolume}
We illustrate why $\det(\mV^\top \mV)$ represents the squared volume using $n=2$ and $n=3$ as examples (see Figure~\ref{fig:Parallelepiped}).  For $n=2$, where $\mV = [\vv_1\ \vv_2]$, the volume of the parallelepiped (equivalently, the area of the parallelogram) is simply $\sin(\theta)$, where $\theta$ is the angle between $\vv_1$ and $\vv_2$. In this case, 
\[
\mV^\top \mV =  
\begin{bmatrix}
    \vv_1^\top \vv_1  &\vv_1^\top \vv_2 \\
    \vv_2^\top \vv_1  &\vv_2^\top \vv_2
\end{bmatrix}
=
\begin{bmatrix}
    1  &\cos\theta \\
     \cos\theta &1 
\end{bmatrix}.
\]
Therefore, $\sqrt{\det(\mV^\top \mV)} = \sqrt{1 - \cos^2 \theta} = \sin \theta$. Similarly, for $\mV = [\vv_1\ \vv_2\ \vv_3]$, the $\sqrt{\det(\mV^\top \mV)}$ exactly computes $|\vv_1^\top (\vv_2 \times \vv_3)|$, which is the volume of the parallelepiped.

\begin{figure}[ht]
    \centering
    \begin{minipage}{0.48\linewidth}
        \centering
        \includegraphics[width=\linewidth]{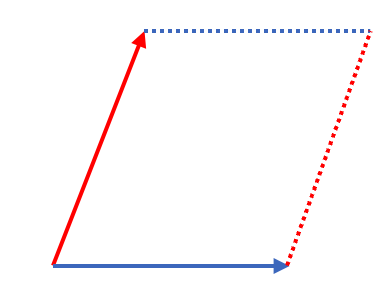}
    \end{minipage}
    \begin{minipage}{0.48\linewidth}
        \centering
        \includegraphics[width=\linewidth]{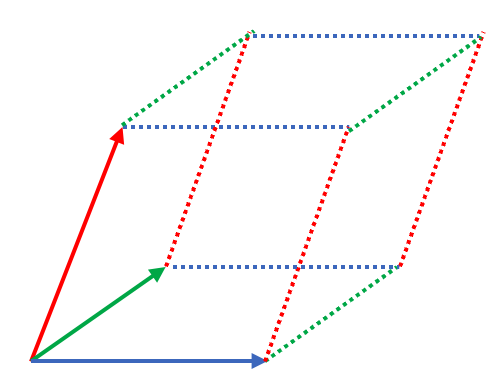}
    \end{minipage}
    \caption{Parallelogram for $n=2$ (left) and parallelepiped for $n=3$ (right)}
    \label{fig:Parallelepiped}
\end{figure}

Beyond its geometric interpretation, $\det(\mV^\top \mV)$ also quantifies the orthogonality of the vectors and is closely related to the determinantal point process (DPP) \cite{kulesza2012determinantal}. DPPs favor diverse or orthogonal subsets by assigning higher probabilities to sets with dissimilar elements, using determinants to model repulsion among points. The Gram matrix $\mV^\top \mV$ is commonly used as the kernel matrix in DPPs, where the determinant of its submatrices determines subset probabilities. For a more detailed discussion, see \cite{li2024rule} and \cite{kulesza2012determinantal}. Moreover, some prior work has considered a similar volume measure to quantify vector diversity and orthogonality in LLM topics, such as the second orthogonality measure in \cite{li2024rule}, which motivates our design of semantic volume.

\section{Proofs} \label{sec:Appendix-Proofs}

\subsection{Proof of Proposition~\ref{prop:OptimalTauFormula}}

\begin{proof}
The formula \eqref{eq:optimal_tau} directly follows from the definition of $F_1$ scores. Recall we have
\[
\text{Precision}(\tau) = \frac{\text{TP}(\tau)}{\text{TP}(\tau) + \text{FP}(\tau)}, \quad
\text{Recall}(\tau) = \frac{\text{TP}(\tau)}{\text{TP}(\tau) + \text{FN}(\tau)},
\]
where $\text{TP}(\tau) = \sum_{i=1}^M \mathbf{1}_{\hat{y}_i(\tau) = 1 \land y_i = 1}$, $\text{FP}(\tau) = \sum_{i=1}^M \mathbf{1}_{\hat{y}_i(\tau) = 1 \land y_i = 0}$, and $\text{FN}(\tau) = \sum_{i=1}^M \mathbf{1}_{\hat{y}_i(\tau) = 0 \land y_i = 1}$. $F_1$ score is the harmonic mean of Precision and Recall:
\begin{align*}
    &F_1(\tau) 
    = 2 \cdot \frac{\text{Precision}(\tau) \times \text{Recall}(\tau)}{\text{Precision}(\tau) + \text{Recall}(\tau)}\\
    =& \frac{2 \cdot \text{TP}(\tau)}{2 \cdot \text{TP}(\tau) + \text{FP}(\tau) + \text{FN}(\tau)}\\
    =& \frac{2 \sum_{i=1}^M \mathbf{1}_{\hat{y}_i(\tau) = y_i = 1} }
    {
    2\sum_{i=1}^M \mathbf{1}_{\hat{y}_i(\tau) = 1 \land y_i = 1}  
    + \sum_{i=1}^M \mathbf{1}_{\hat{y}_i(\tau) = 1 \land y_i = 0}
    + \sum_{i=1}^M \mathbf{1}_{\hat{y}_i(\tau) = 0 \land y_i = 1}
    }\\
    =& \frac{2 \sum_{i=1}^M \mathbf{1}_{\hat{y}_i(\tau) = y_i = 1} }{\sum_{i=1}^M ( \mathbf{1}_{\hat{y}_i(\tau) = 1} + \mathbf{1}_{y_i=1} ) } .
\end{align*}
\end{proof}

\subsection{Proof of Theorem~\ref{thm:DiffEntropy}}

First, in Lemma~\ref{lem:LinearTransformationInvariance}, we formally justify that our semantic volume measure for uncertainty detection is invariant under linear transformations applied to the uncertainty measure. This invariance follows directly from the step of searching for optimal threshold $ \tau^* $ in our method. 

\begin{lemma}[Invariance of semantic volume method under linear transformation] \label{lem:LinearTransformationInvariance}
Our method is invariant under linear transformation on the semantic volume measure $m(\cdot)$. More precisely, 
denote $\tau^\ast$ as the decision boundary defined in \eqref{eq:optimal_tau}, we label 
\[
y_i \bydef \mathbf{1}_{m(s_i) < \tau^\ast}.
\]
Given any linear transformation $T(m_i) \bydef \alpha m_i + \beta$ ($\alpha > 0$) applied to the measure $m(\cdot)$, denote the new labels of our method under the new measure $T(m(\cdot))$ as 
\[
\tilde{y}_i \bydef \mathbf{1}_{T(m(s_i)) < \tilde{\tau}^\ast}.
\]
Then we have
\[
\tilde{y}_i = y_i \quad \text{for all } s_i \in \calD \setminus \calL.
\]
\end{lemma}

\begin{proof}
If we transform $m_i$ to $\tilde{m}_i = T(m_i)$, the new threshold $\tilde{\tau}^*$ that maximizes the same metric still yields the same partition of the labeled samples. In other words,
\[
m_i > \tau^* \iff T(m_i) = \tilde{m}_i > T(\tau^*),
\]
a result of the simple fact that $m_i > \tau^* \iff \alpha(m_i) + \beta > \alpha(\tau^*) + \beta$. Thus our new threshold $\tilde{\tau}^*$ is effectively $T(\tau^*)$. Note that here if $\alpha < 0$ in $T(m_i) \bydef \alpha m_i + \beta$,  then naturally we should flip the decisions of our method, and then the labels still remain the same.
\end{proof}

Lemma~\ref{lem:GaussianEntropy} below provides the formula for the differential entropy of Gaussian vectors. This is a known result but we include this lemma and provide proof here for completeness.
\begin{lemma}[Differential entropy of Gaussian vectors]\label{lem:GaussianEntropy}
For $X \in \R^d$ that follows a multivariate normal distribution
\[
X \sim \mathcal{N}(\bm{\mu}, \mSigma).
\]
The differential entropy of $X$ is given by
\[
\calH(\vx) =  \frac{1}{2} (\log \det(\mSigma) + d \log(2\pi)  + d).
\]
\end{lemma}

\begin{proof}
By the definition of the differential entropy and the probability density function of the multi-variate normal distribution, we have
\begin{align*}
   &\calH (X) = -\int_{X} p(x) \log p(x)  dx = -\E_{x\sim X} \left[ \log p(x) \right] \\
   & = -\E \log \left[ \frac{1}{\sqrt{(2\pi)^{d} \det(\mSigma)}}  \exp \left( -\frac{1}{2} (\vx - \bm{\mu})^\top \mSigma^{-1} (\vx - \bm{\mu}) \right) \right]\\
   & = \E \log \sqrt{(2\pi)^{d} \det(\mSigma)}   + \E \left[ \frac{1}{2} (\vx - \bm{\mu})^\top \mSigma^{-1} (\vx - \bm{\mu})\right]\\
   & = \frac{1}{2}\E \left[ d \log (2\pi)  + \log \det(\mSigma) \right]  + \frac{1}{2} \E \left[  (\vx - \bm{\mu})^\top \mSigma^{-1} (\vx - \bm{\mu})\right]\\
   & = \frac{1}{2} \left( 
    d \log (2\pi)  + \log \det(\mSigma)  + \E \left[  (\vx - \bm{\mu})^\top \mSigma^{-1} (\vx - \bm{\mu})\right]
   \right)
\end{align*}
By the cyclic property of trace, we have
\begin{align*}
     \E \left[  (\vx - \bm{\mu})^\top \mSigma^{-1} (\vx - \bm{\mu})\right]
    &= \E \Tr \left[(\vx - \bm{\mu})^\top \mSigma^{-1} (\vx - \bm{\mu})\right]\\
    &= \E \Tr \left[ \mSigma^{-1} (\vx - \bm{\mu}) (\vx - \bm{\mu})^\top \right]\\
    &= \Tr \mSigma^{-1} \E\left[ (\vx - \bm{\mu}) (\vx - \bm{\mu})^\top \right]\\
    &= \Tr \mSigma^{-1} \mSigma\\
    &= d.
\end{align*}
Combine these steps, we get 
\[
  \calH (X) = \frac{1}{2} (\log \det(\mSigma) + d \log(2\pi)  + d).
\]
\end{proof}

\begin{lemma}[Eigenvalue distribution of rank-one perturbed matrix]\label{lem:Perturb-Logdot}
For non-singular matrix $\mM \in \R^{d\times d}$ and any rank-one perturbation $\mS \bydef \mM + \vu\vv^\top$, we have
\begin{equation}\label{eq:Perturb-Logdet}
    \log \det(\mS) = \log \det(\mM) + \log (1 + \vv^\top \mM^{-1} \vu),
\end{equation}
\end{lemma}

\begin{proof}
By the matrix determinant lemma \cite{harville1998matrix}, we have
\begin{align*}
    &\det(\mS) = \det( \mM + \vu\vv^\top) = \det(\mM) (1 + \vv^\top \mM^{-1} \vu)\\
    \Longrightarrow & \log\det(\mS) = \log \det(\mM) + \log (1 + \vv^\top \mM^{-1} \vu).
\end{align*}
\end{proof}

Next, we combine the results above to derive the proof of Theorem~\ref{thm:DiffEntropy}.
\begin{proof}[Proof of Theorem~\ref{thm:DiffEntropy}]
For each $s$ (either a query or response) with embedding $\bvx \in \R^d$, we extend it to get $n$ perturbations $\vx_1, \vx_2, \dots, \vx_n$. Equivalently, $\{\vv_i\}_{i \in [n]}$ are perturbations of $\bvv$. Assume the perturbation is Gaussian with covariance matrix $\mSigma$. Then $\{\vx_i\}_{i \in [n]}$ are samples for a Gaussian vector $X \sim \calN(\bvx, \mSigma)$.  Denote $\lambda_i(\mA)$ as the $i$-th largest eigenvalue of matrix $\mA$. Note that  $\mV^\top \mV$ and $\mV \mV^\top$ are both positive semi-definite matrices and share the same nonzero eigenvalues (can be easily proved using SVD of $\mV$). Similarly for $\mX^\top \mX$ and $\mX\mX^\top$. We can diagonalize $\mSigma$ as $\mSigma = Q \Lambda Q^\top$ where $Q$ is an orthogonal matrix and $\Lambda = \text{diag}(\lambda_1, \dots, \lambda_d)$. Define an isotropic Gaussian vector $Z \sim \calN(0, \mI_d)$. Then we can write $X = Q \Lambda^{1/2} Z$. Hence we have
\begin{align*}
    &\E\|X\|^2 = \E[ X^\top X] =  \E [Z^\top \Lambda^{1/2} Q^\top Q \Lambda^{1/2} Z] \\
 &= \E[ Z^\top \Lambda Z] = \E \left[ \sum_{i=1}^d \lambda_i Z_i^2 \right] =  \sum_{i=1}^d \lambda_i \E [Z_i^2 ]  = \Tr(\mSigma).
\end{align*}
Note that in high-dimensions,  $\|X\|^2$ is concentrated around its mean $\E\|X\|^2 = \Tr(\mSigma)$ (particularly for isotropic Gaussian, this is $\sqrt{D}$). Under this regime, we have $\mV \bydef  [\vv_1, \vv_2, \dots, \vv_n] \approx \frac{1}{\sqrt{\Tr(\mSigma)}}\mX$. Let $\alpha = n/\Tr(\mSigma)$. Then
\begin{align*}
     &\log \det(\mV^\top \mV) 
     \approx \log \det\left(\frac{1}{\Tr(\mSigma)}\mX^\top \mX\right) \\
     &= \log \det\left(\alpha \cdot \frac{1}{n}\mX^\top \mX\right) \\
     &= \log \det\left(\frac{1}{n}\mX^\top \mX\right)  + n \log \alpha .
\end{align*}

Since our method sets a threshold $\tau^*$ on the log-determinant value, searched using a held-out set $\calL \in \calD$, the final classification is invariant to shifts (and more generally, to any linear transformation. See Proposition~\ref{lem:LinearTransformationInvariance}). Hence we can ignore the constant term $n \log \alpha $ and only focus on $\log \det\left(\frac{1}{n}\mX^\top \mX\right)$.  We have

\begin{align*}
    &\log \det\left(\frac{1}{n}\mX^\top \mX\right)  = \log \prod_{i=1}^n \lambda_i \left(\frac{1}{n}\mX^\top \mX\right)\\
    &= \sum_{i=1}^n \log \lambda_i \left(\frac{1}{n}\mX^\top \mX\right) = \sum_{i=1}^n \log \lambda_i \left(\frac{1}{n}\mX \mX^\top \right).
\end{align*}
Thus, it suffices to study the eigenvalues of $\frac{1}{n}\mX \mX^\top$. Note that we have
\begin{align*}
\mX \mX^\top - n \mSigma 
    &= \left(\sum_{i=1}^n \vx_i \vx_i^\top\right) - \left(\sum_{i=1}^n (\vx_i-\bvx) (\vx_i-\bvx)^\top \right) \\
    &= \sum_{i=1}^n (\bvx \vx_i^\top - \vx_i \bvx^\top +  \bvx \bvx^\top)  \\
    &=  \bvx \left(\sum_{i=1}^n \vx_i\right)^\top - \left(\sum_{i=1}^n\vx_i\right) \bvx^\top + n \bvx \bvx^\top)  \\
    &\approx  n \bvx \bvx^\top - n \bvx \bvx^\top + n \bvx \bvx^\top  \\
    &= n \bvx \bvx^\top.
\end{align*}
This implies that
\begin{equation}\label{eq:rank-1-perturbation}
    \frac{1}{n}\mX \mX^\top - \mSigma  \approx  \bvx \bvx^\top.
\end{equation}
Thus $\mSigma$ is essentially a rank-one perturbation of $\frac{1}{n}\mX \mX^\top$. By Lemma~\ref{lem:GaussianEntropy}, we know the differential entropy of $X$ is:
\[
  \calH (X) = \frac{1}{2} (\log \det(\mSigma) + d \log(2\pi)  + d).
\]
Again due to the invariance of our method to linear transformation, we can only focus on $\log \det(\mSigma)$. When $d$ is large, by Lemma~\ref{lem:Perturb-Logdot}, 
let $\mM \bydef \frac{1}{n}\mX \mX^\top$, we have $\log \det(\mSigma) = \log \det \mM  + \log (1 + \bvx^\top \mM^{-1} \bvx^\top )$. Note here $\bvx$ is constant, and recall that $\mM$ concentrate around a deterministic matrix 
$\mSigma + \bvx \bvx^\top$ from \eqref{eq:rank-1-perturbation} above (one can show $\|\mM  - (\mSigma + \bvx \bvx^\top)\| \overset{p}{\rightarrow} 0 $ as $n\to\infty$, using standard random matrix theory tools). Therefore combining all above, we have derived that $\log \det \mM $ is essentially shifted $\log \det(\mSigma)$ in high-dimensional regimes, and hence same for $\log \det \frac{1}{n}\mX^\top \mX$. This completes the proof.
\end{proof}

\subsection{Gaussian Assumption}
Our theoretical analysis assumes a Gaussian distribution for the embedding vectors, particularly after PCA projection. This Gaussian assumption facilitates the interpretation of our proposed semantic volume measure in terms of differential entropy. Here, we provide empirical evidence to support the relative Gaussianity of these projected embedding vectors.

To empirically test Gaussianity, we leverage the fact that if samples  follow a multivariate normal distribution with mean  and covariance matrix, then the \textit{squared Mahalanobis distances}:
\[
D_i^2 = (x_i - \mu)^T \Sigma^{-1} (x_i - \mu)
\]
should follow a \textit{chi-square distribution} with $d$ degrees of freedom:
\[
D_i^2 \sim \chi^2_d
\]
Thus we compute $D_i^2$ for each sample, using the empirical mean and covariance matrix of the perturbation vectors, and compare their distribution to the $\chi^2_d$ distribution. We analyze the quantiles of Mahalanobis distances and the theoretical $\chi^2_d$ distribution. Better alignment with the line $y=x$ in these plots indicates stronger evidence of Gaussianity. Example Q–Q plots are shown in Figure~\ref{fig:QQ-plot}, where the empirical distributions generally align well with the theoretical expectations.

\begin{figure}[H]
    \centering
    \begin{subfigure}[b]{0.48\linewidth}
        \includegraphics[width=\textwidth]{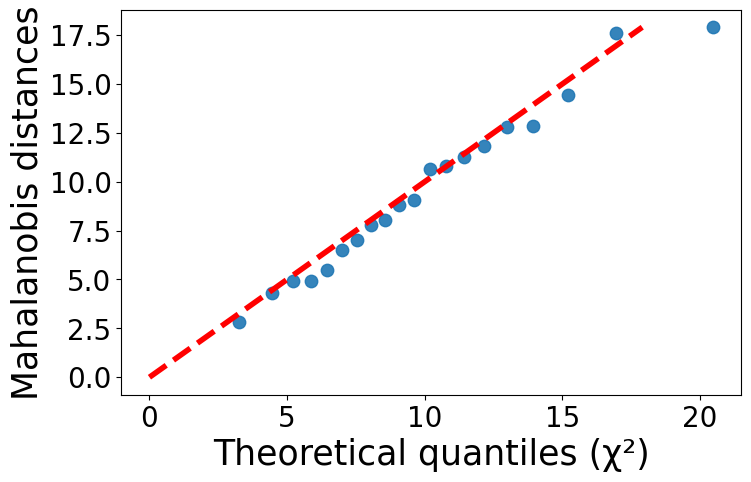}
    \end{subfigure}
    \begin{subfigure}[b]{0.48\linewidth}
        \includegraphics[width=\textwidth]{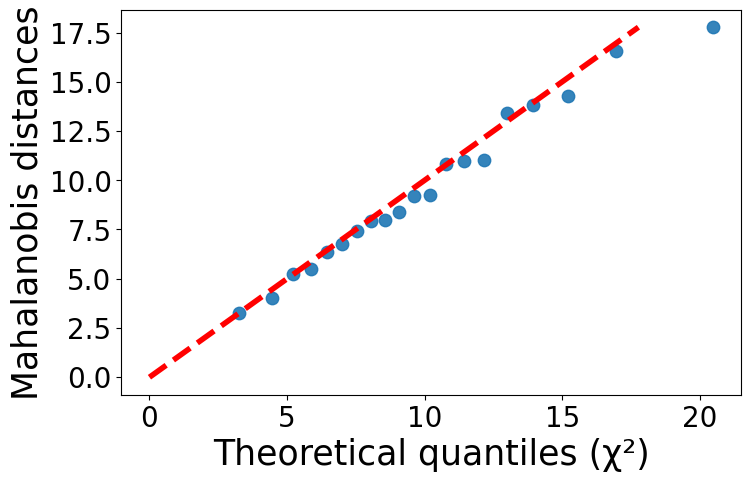}
    \end{subfigure}
    \begin{subfigure}[b]{0.48\linewidth}
        \includegraphics[width=\textwidth]{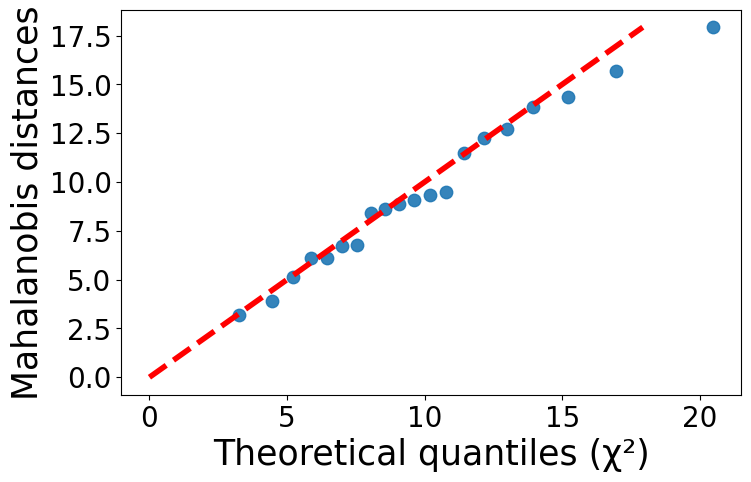}
    \end{subfigure}
    \begin{subfigure}[b]{0.48\linewidth}
        \includegraphics[width=\textwidth]{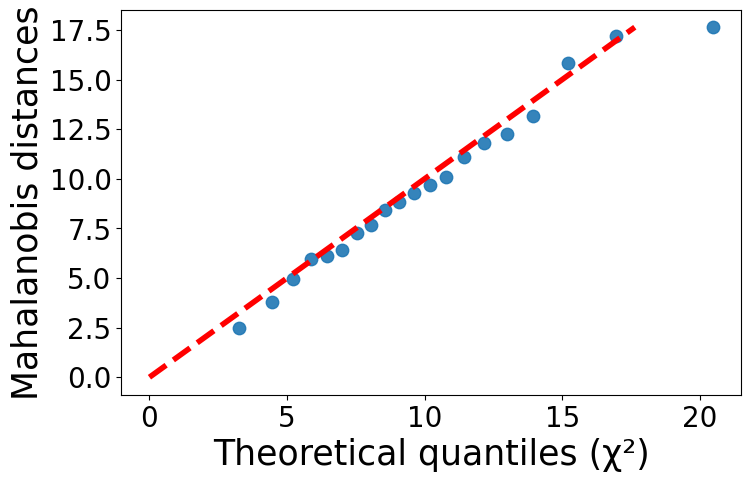}
    \end{subfigure}

    \vspace{0.5em}
    \begin{subfigure}[b]{0.48\linewidth}
        \includegraphics[width=\textwidth]{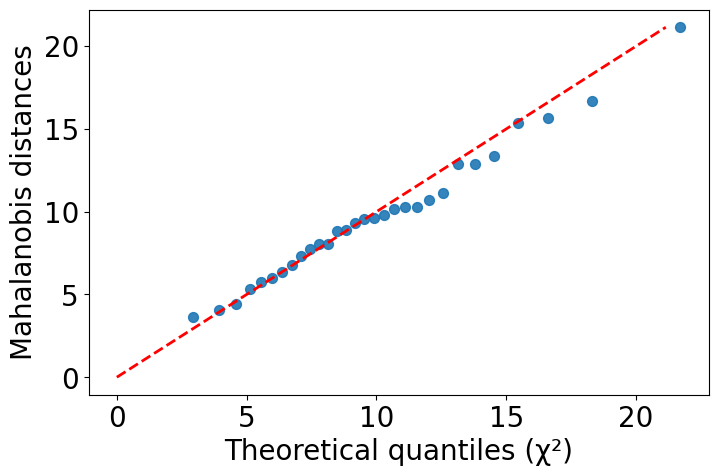}
    \end{subfigure}
    \begin{subfigure}[b]{0.48\linewidth}
        \includegraphics[width=\textwidth]{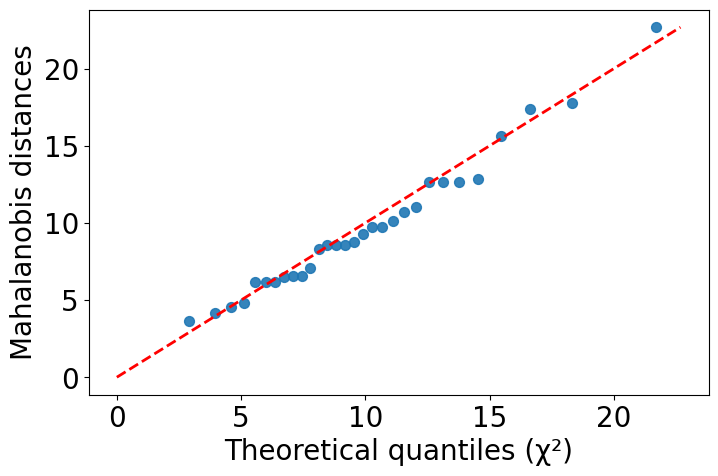}
    \end{subfigure}
    \begin{subfigure}[b]{0.48\linewidth}
        \includegraphics[width=\textwidth]{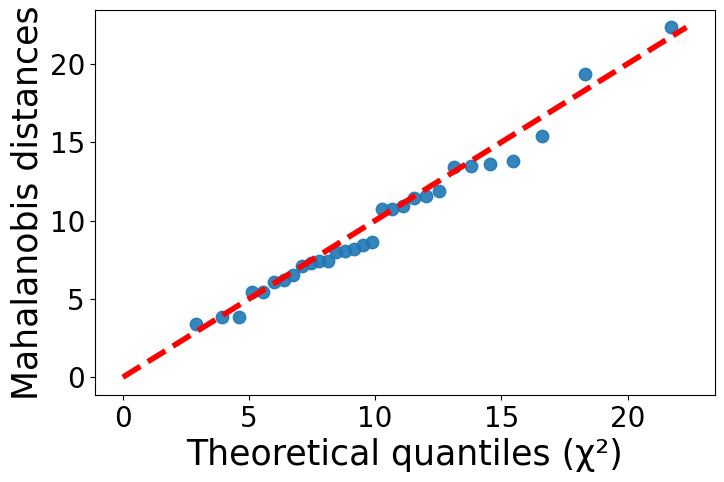}
    \end{subfigure}
    \begin{subfigure}[b]{0.48\linewidth}
        \includegraphics[width=\textwidth]{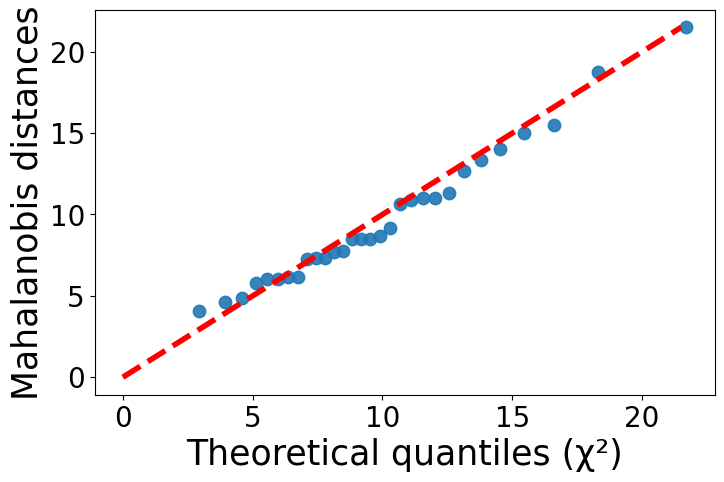}
    \end{subfigure}
    \caption{Mahalanobis–chi-square Q–Q plots for four examples each of queries (top row) and responses (bottom row).}
    \label{fig:QQ-plot}
\end{figure}

To quantitatively evaluate Gaussianity, we compute the $R^2$ score from a linear regression between the empirical quantiles of squared Mahalanobis distances and the theoretical $\chi^2_d$ quantiles:
\[
R^2 = 1 - \frac{\sum_{i=1}^{n}(y_i - \hat{y}_i)^2}{\sum_{i=1}^{n}(y_i - \bar{y})^2},
\]
where $y_i$ are the observed values, $\hat{y}_i$ the theoretical values, and $\bar{y} = \frac{1}{n} \sum_{i=1}^{n} y_i$ the mean of the observations. Note that $R^2 \in (-\infty, 1]$ and a larger $R^2$ value indicates better alignment and thus stronger Gaussianity. The histogram of $R^2$ values is plotted in Figure~\ref{fig:QQ-plot-R2}. If we use 0.8 as a threshold for high Gaussianity, 90.5\% of the query embeddings (in CLAMBER) and 71.4\% of the response embeddings (in TriviaQA) exceed this threshold, suggesting strong empirical support for our Gaussian assumption.

Although these distributions are not perfectly Gaussian in practice, our theoretical results primarily serve as interpretative tools to understand how semantic embedding measures semantic dispersion. We leave the further theoretical exploration on this to future work.

\begin{figure}[H]
    \centering
    \begin{subfigure}[b]{0.48\linewidth}
        \includegraphics[width=\linewidth]{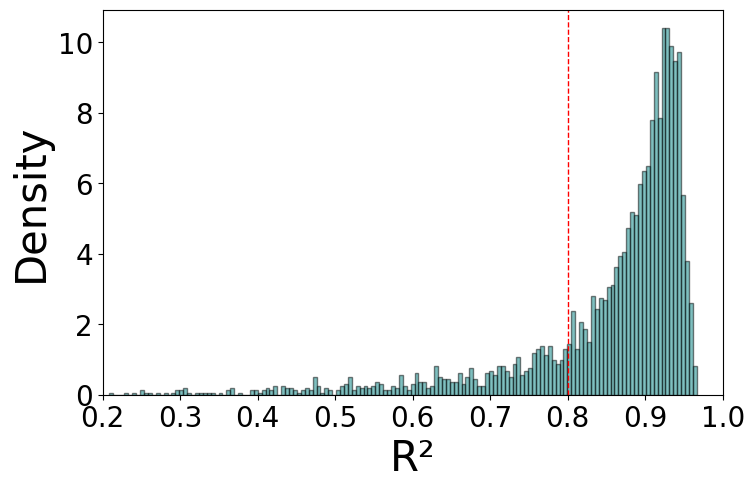}
    \end{subfigure}
    \begin{subfigure}[b]{0.48\linewidth}
        \includegraphics[width=\linewidth]{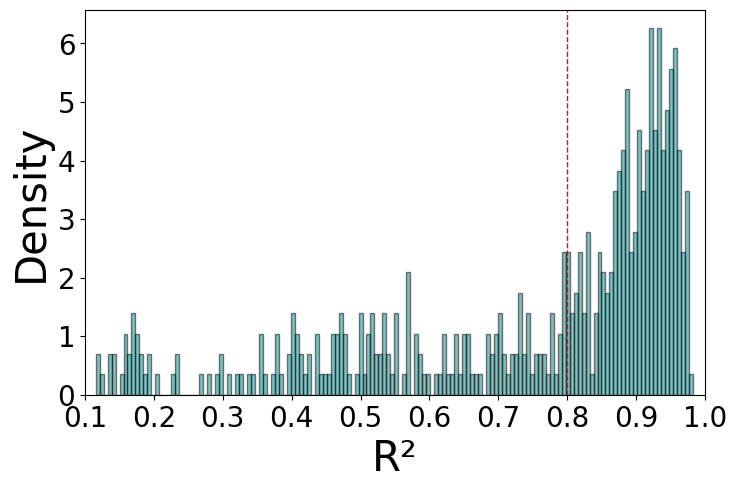}
    \end{subfigure}
    \caption{Distribution of $R^2$.}
    \label{fig:QQ-plot-R2}
\end{figure}

\section{Numerical Stability with $\epsilon$} \label{sec:Appendix-Epsilon}
Here, we present the distribution of the matrix norm $\|\mV^\top \mV\|$ (spectral norm) across our datasets in Figure~\ref{fig:Epsilon-MatrixNormDistribution}. The histograms indicate that the chosen value of $\epsilon = 10^{-10}$ is negligible compared to the typical magnitude of $\|\mV^\top \mV\|$. This confirms that $\epsilon$ only serves to ensure numerical stability rather than influencing the quantification. Moreover, in practice, we observe that exact repetition among sampled perturbations is rare.

\begin{figure}[H]
    \centering
    \begin{minipage}{0.48\linewidth}
        \centering
        \includegraphics[width=\linewidth]{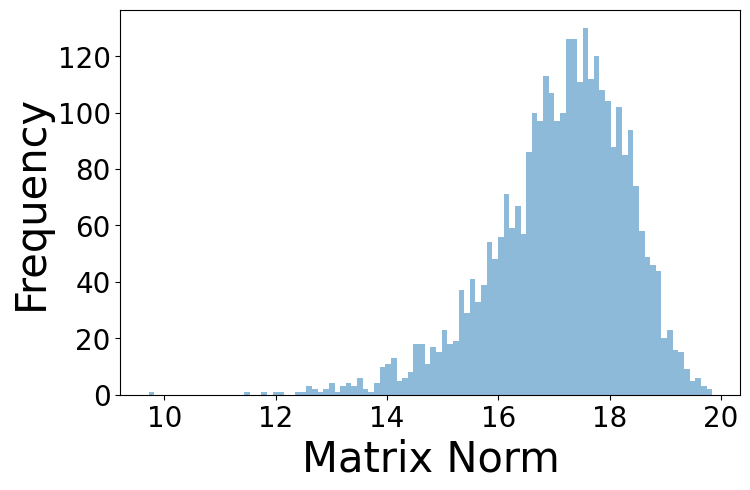}
    \end{minipage}
    \begin{minipage}{0.48\linewidth}
        \centering
        \includegraphics[width=\linewidth]{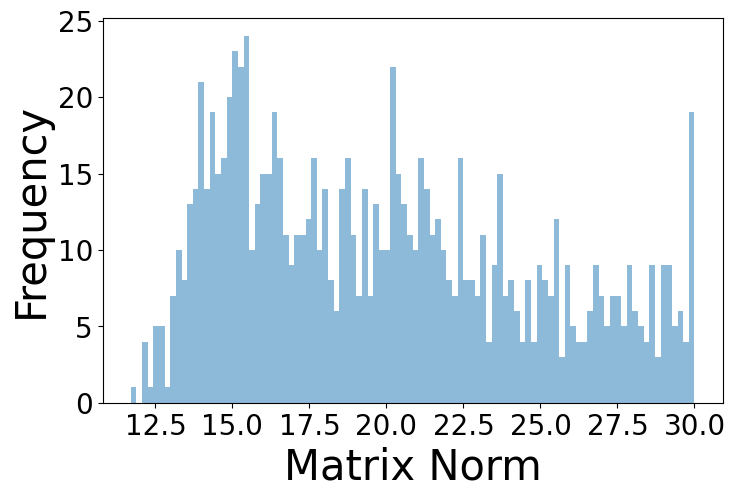}
    \end{minipage}
    \caption{Distribution of the spectral norm  $\|\mV^\top \mV\|$ for the perturbation embedding vectors (the left histogram corresponds to query perturbations and the right is for response perturbations.}
    \label{fig:Epsilon-MatrixNormDistribution}
\end{figure}

\section{Variation: Reduced Dimension $d$} \label{sec:Appendix-Variation-d}
In Figure~\ref{fig:Ablation}, we illustrate the impact of varying $d$ on performance. Additionally, we compare our method with the Semantic Volume that uses the original embedding vectors instead of the projected ones. Our results show that projecting to a lower-dimensional space improves performance, possibly because PCA projections better separate perturbation vectors. Furthermore, our study suggests that there may be an optimal, task-dependent choice of $d$. Exploring this further and characterizing the optimal $d$ is left for future work.

Moreover, alternative dimensionality reduction techniques, such as autoencoders, may also be employed. Investigating the impact of different dimensionality reduction methods remains an important direction for future research.
\begin{figure}[ht]
    \centering
    \begin{minipage}{0.48\linewidth}
        \centering
        \includegraphics[width=1.0\linewidth]{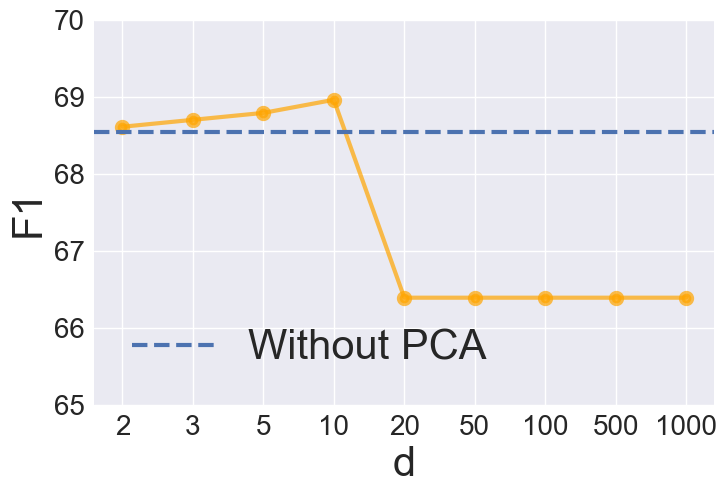}
    \end{minipage}
    \begin{minipage}{0.48\linewidth}
        \centering
        \includegraphics[width=1.0\linewidth]{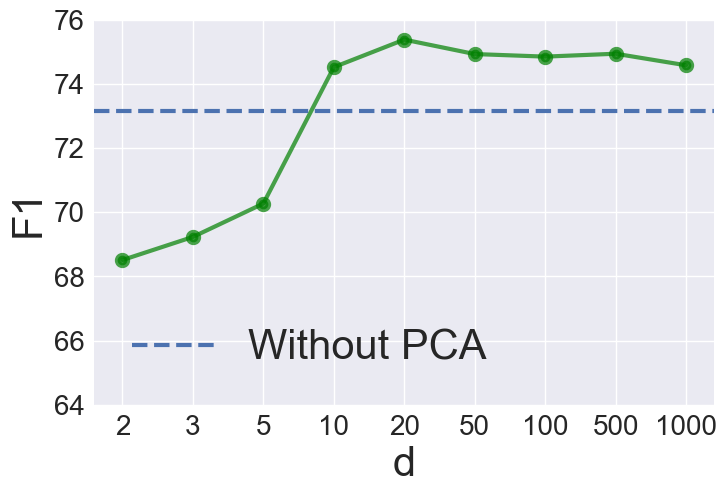}
    \end{minipage}
    \caption{F1 scores for different dimension $d$ in PCA dimension reduction. The dashed line represents the F1 score of the semantic volume method without dimension reduction. The left figure corresponds to the external uncertainty task, while the right figure corresponds to the internal uncertainty task.}
    \label{fig:Ablation}
\end{figure}

\section{Variation: Perturbation Sample Size $n$} \label{sec:Appendix-Variation-n}
The perturbation sample size $n$ is a hyperparameter in our method. A larger $n$ provides a more comprehensive estimate of semantic dispersion but increases computational cost, while a smaller $n$ may fail to capture sufficient variation, reducing the reliability of uncertainty quantification.

Figure~\ref{fig:Variation-n} illustrates how varying $n$ affects classification performance. In general, increasing $n$ tends to improve F1 scores, likely due to a more accurate estimation of dispersion. However, beyond a certain point, the improvements diminish while the cost continues to rise. To balance efficiency and performance, we set $n=20$ for external uncertainty detection and internal uncertainty detection.
\begin{figure}[H]
    \centering
    \begin{subfigure}{0.47\linewidth}
        \centering
        \includegraphics[width=1.0\linewidth]{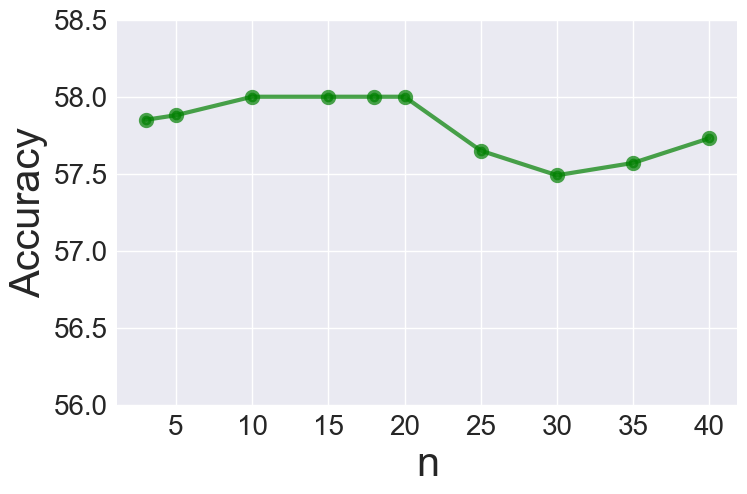}
        \caption{External uncertainty: best accuracy for different $n$.}
    \end{subfigure}
    \begin{subfigure}{0.47\linewidth}
        \centering
        \includegraphics[width=1.0\linewidth]{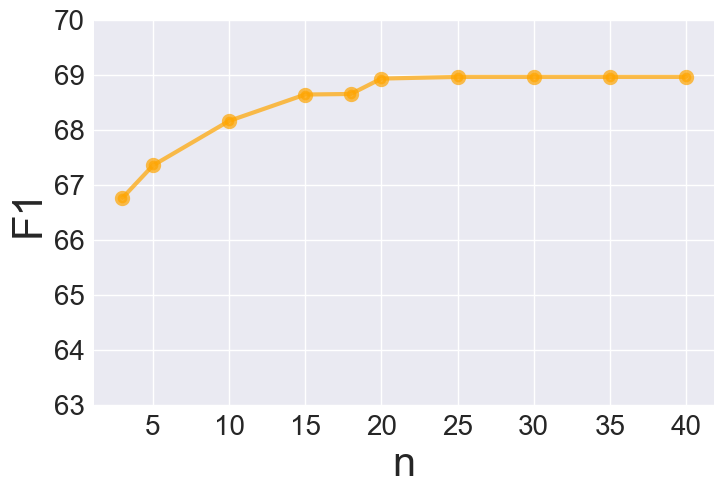}
        \caption{External uncertainty: best F1-score for different $n$.}
    \end{subfigure} \hfill \\[1ex]
    \begin{subfigure}{0.47\linewidth}
        \centering
        \includegraphics[width=1.0\linewidth]{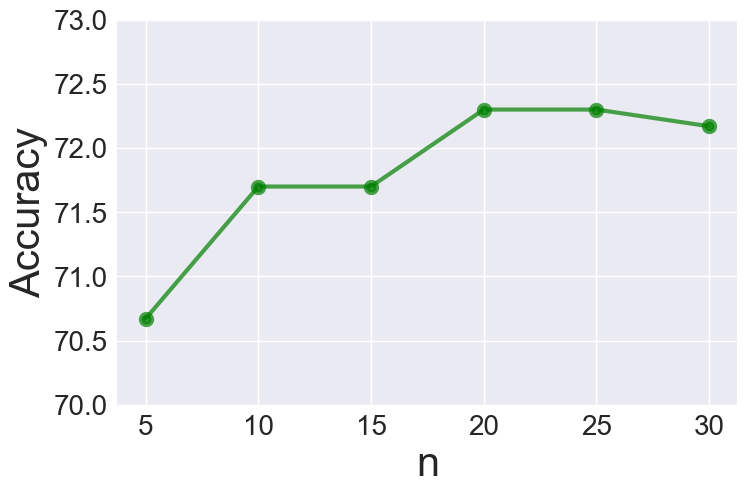}
        \caption{Internal uncertainty: best accuracy for different $n$.}
    \end{subfigure}
    \begin{subfigure}{0.47\linewidth}
        \centering
        \includegraphics[width=1.0\linewidth]{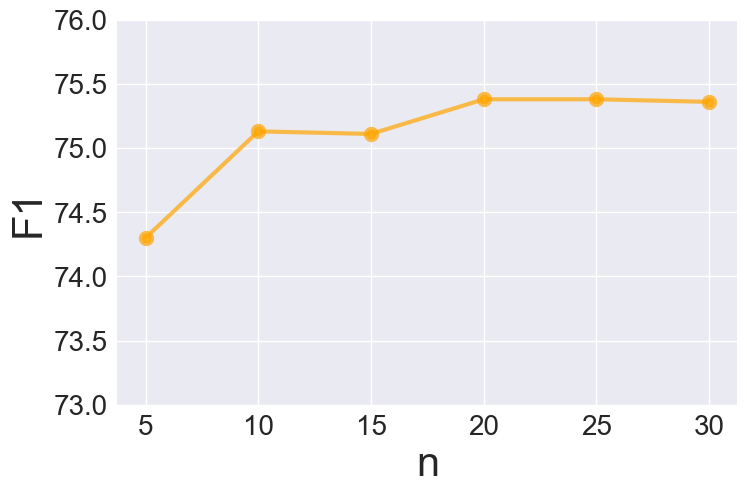}
        \caption{Internal uncertainty: best F1-score for different $n$.}
    \end{subfigure} \hfill
    \caption{}
    \label{fig:Variation-n}
\end{figure}

\section{Variation: Embedding Models} \label{sec:Appendix-Variation-EmbeddingModels}

To examine the effect of different embedding models on our method, we explore two larger sentence-transformers: \texttt{Alibaba-NLP/gte-Qwen2-7B-instruct} (embedding dimension 3584) from the same Qwen family as used in our main results, and \texttt{nvidia/NV-Embed-v2} (embedding dimension 4096) from a different model family. Using the external uncertainty task as a case study, we present the results in Table~\ref{tab:Appendix-Variation-EmbeddingModels}. Our findings indicate that models from different families but with similar sizes produce comparable results. Compared to the Qwen-1.5B sentence transformer used in the main results, performance remains largely similar—while the 7B model achieves slightly higher accuracy but a lower F1 score, this difference may not be statistically significant.

\begin{table}[h]
    \centering
    \renewcommand{\arraystretch}{1.1} 
    \begin{tabular}{l|c|c|c|c}
        \hline
        & \multicolumn{2}{c|}{\textbf{Qwen2-7B}} & \multicolumn{2}{c}{\textbf{NV-Embed}} \\
        \hline
        & \textbf{Accuracy} & \textbf{F1} & \textbf{Accuracy} & \textbf{F1} \\
        \hline
        $n=25$ & $57.8$ & $68.0$ & $57.4$ & $67.5$ \\
        $n=20$ & $58.1$ & $68.0$ & $58.4$ & $67.9$ \\
        $n=15$ & $57.7$ & $68.1$ & $57.9$ & $67.8$ \\
        $n=10$ & $57.2$ & $67.7$ & $57.4$ & $67.5$ \\
        $n=5$  & $56.8$ & $66.9$ & $57.7$ & $67.0$ \\
        \hline
    \end{tabular}
    \caption{Performance with different embedding models.}
    \label{tab:Appendix-Variation-EmbeddingModels}
\end{table}

\section{Variation: Response Generation Model} \label{sec:Appendix-Variation-LargeResponseGenerationModel}
To assess the generalizability of our method, we conduct the experiment for internal uncertainty detection for these response models:
\begin{itemize}
    \item \texttt{Llama3.2-\allowbreak 1B-\allowbreak Instruct} (Table~\ref{tab:Results-Internal})
    \item \texttt{Llama3-8B-Instruct} (Table~\ref{tab:Results-Internal-Llama3-8B-Instruct})
    \item \texttt{Qwen2.5-\allowbreak 1.5B-\allowbreak Instruct} (Table~\ref{tab:Results-Internal-Qwen2.5-1.5B-Instruct})
    \item \texttt{Qwen3-\allowbreak 8B} (Table~\ref{tab:Results-Internal-Qwen3-8B})
    \item \texttt{Qwen3-\allowbreak 14B} (Table~\ref{tab:Results-Internal-Qwen3-14B}).
    \item \texttt{Mistral-\allowbreak 7B-\allowbreak Instruct} (Table~\ref{tab:Results-Internal-Mistral-7B-Instruct}).
\end{itemize}
Notably, our methods consistently outperforms the baselines for all the response model we tested. It is also worth noting that, with a larger model such as \texttt{Llama3-8B-Instruct}, the probability-based methods show a lot higher latency. In our experiment, both Last Token Entropy and Log Probabilities took more than 4 hours on a H100-80GB GPU for each response model.

Furthermore, similar to Table~\ref{tab:Results-Internal}, we observe that methods relying on prompting the LLM (including pTrue) sometimes exhibit a strong bias toward predicting almost exclusively ``Yes'' or ``No''. In Table~\ref{tab:Results-Internal-Llama3-8B-Instruct}, for example, \textit{Prompt Llama3.2-1B-Instruct} and \textit{pTrue (Llama3.2-1B-Instruct)} predict nearly all samples as ``Yes'', whereas \textit{pTrue (Mistral-7B-Instruct)} predicts almost all as ``No''. This instability highlights another drawback of using prompted LLMs for uncertainty prediction.

\begin{table*}[t!]
\centering
\renewcommand{\arraystretch}{1.28}   
\resizebox{0.72\linewidth}{!}{%
\begin{tabular}{l|ccc|ccc}
\hline
\multirow{2}{*}{\textbf{Method}} &
\multicolumn{3}{c|}{\textbf{TriviaQA}} &
\multicolumn{3}{c}{\textbf{SQuAD}} \\ \cline{2-7}
& \textbf{Acc.} & \textbf{F1} & \textbf{AUROC} & \textbf{Acc.} & \textbf{F1} & \textbf{AUROC} \\ \hline
Prompt Llama3.2-1B-Instruct     &  50.1 &  66.5 &  N/A &  50.8 &  66.6 &  N/A \\
Prompt Llama3-8B-Instruct        &  65.5 &  64.0 &  N/A &  58.3 &  51.6 &  N/A \\
Prompt Mistral-7B-Instruct       &  60.4 &  54.3 &  N/A &  50.1 &  0.6 &  N/A \\
\hline
LastTokenEntropy                 &  56.2 &  55.9 &  63.9 &  51.7 &  35.8 &  53.0 \\
LogProb                          &  53.4 &  51.1 &  54.4 &  57.1 &  55.8 &  61.0 \\ \hline
pTrue Llama3.2-1B-Instruct      &  50.0$_{0.0}$ &  66.6$_{0.0}$ &  55.1$_{0.8}$ &  50.2$_{0.6}$ &  54.9$_{0.6}$ &  51.1$_{0.2}$ \\
pTrue Llama3-8B-Instruct         &  55.2$_{0.2}$ &  23.2$_{7.4}$ &  52.8$_{1.0}$ &  50.7$_{1.3}$ &  57.7$_{1.2}$ &  50.6$_{0.5}$ \\
pTrue Mistral-7B-Instruct        &  50.8$_{1.2}$ &  4.3$_{3.2}$ &  72.7$_{1.6}$ &  51.4$_{0.4}$ &  58.5$_{0.2}$ &  48.9$_{0.3}$ \\
LexicalSimilarity                &  64.2$_{0.2}$ &  72.0$_{0.4}$ &  78.7$_{0.5}$ &  67.0$_{0.3}$ &  64.0$_{0.9}$ &  70.4$_{0.2}$ \\
SemanticEntropy                  &  63.5$_{0.5}$ &  69.4$_{1.8}$ &  73.2$_{0.4}$ &  62.5$_{0.1}$ &  65.2$_{1.3}$ &  66.1$_{0.3}$ \\ \hline
\textbf{Semantic Volume (ours)} &  \textbf{72.0$_{0.1}$} &  \textbf{74.9$_{0.0}$} &  \textbf{79.7$_{0.2}$} &  \textbf{68.0$_{0.2}$} &  \textbf{71.0$_{0.1}$} &  \textbf{71.0$_{0.9}$} \\ \hline
\end{tabular}}
\caption{Response model: \texttt{Llama3-8B-Instruct}. Internal uncertainty: Accuracy, F1, and AUROC on TriviaQA and SQuAD.  “N/A’’ indicates AUROC is not applicable (no probabilistic scores).  For sampling-based methods, we report mean $\pm$ std.\ over three trials (std.\ shown as subscript).}
\label{tab:Results-Internal-Llama3-8B-Instruct}
\end{table*}

\begin{table*}[t!]
\centering
\renewcommand{\arraystretch}{1.28}   
\resizebox{0.72\linewidth}{!}{%
\begin{tabular}{l|ccc|ccc}
\hline
\multirow{2}{*}{\textbf{Method}} &
\multicolumn{3}{c|}{\textbf{TriviaQA}} &
\multicolumn{3}{c}{\textbf{SQuAD}} \\ \cline{2-7}
& \textbf{Acc.} & \textbf{F1} & \textbf{AUROC} & \textbf{Acc.} & \textbf{F1} & \textbf{AUROC} \\ \hline
Prompt Llama3.2-1B-Instruct     &  50.6 &  66.7 &  N/A &  50.6 &  66.7 &  N/A \\
Prompt Llama3-8B-Instruct        &  65.2 &  62.5 &  N/A &  64.3 &  62.8 &  N/A \\
Prompt Mistral-7B-Instruct       &  51.2 &  5.6 &  N/A &  50.8 &  3.7 &  N/A \\ \hline
LastTokenEntropy                 &  48.3 &  29.9 &  49.3 &  48.8 &  13.0 &  50.2 \\
LogProb                          &  55.5 &  52.7 &  58.4 &  54.2 &  39.8 &  58.1 \\ \hline
pTrue Llama3.2-1B-Instruct      &  53.3$_{0.7}$ &  52.0$_{0.9}$ &  50.7$_{0.7}$ &  51.8$_{0.6}$ &  49.9$_{0.8}$ &  50.5$_{0.9}$ \\
pTrue Llama3-8B-Instruct         &  52.0$_{0.1}$ &  63.0$_{0.1}$ &  49.9$_{0.3}$ &  52.1$_{0.2}$ &  62.9$_{0.3}$ &  50.3$_{1.1}$ \\
pTrue Mistral-7B-Instruct        &  52.8$_{0.7}$ &  54.0$_{0.9}$ &  50.5$_{0.1}$ &  53.6$_{0.1}$ &  55.4$_{0.2}$ &  51.8$_{0.6}$ \\
LexicalSimilarity                &  64.3$_{0.4}$ &  62.4$_{0.7}$ &  66.0$_{0.1}$ &  65.6$_{0.2}$ &  61.0$_{1.8}$ &  65.7$_{0.2}$ \\
SemanticEntropy                  &  59.5$_{0.7}$ &  51.2$_{2.2}$ &  63.7$_{0.0}$ &  58.5$_{0.6}$ &  55.3$_{1.2}$ &  62.3$_{0.5}$ \\ \hline
\textbf{Semantic Volume (ours)} &  \textbf{65.4$_{0.2}$} &  \textbf{68.6$_{0.0}$} &  \textbf{67.6$_{0.8}$} &  \textbf{66.0$_{0.1}$} &  \textbf{69.2$_{0.2}$} &  \textbf{68.7$_{0.7}$} \\ \hline
\end{tabular}}
\caption{Response model: \texttt{Mistral-7B-Instruct}. Internal uncertainty: Accuracy, F1, and AUROC on TriviaQA and SQuAD.  “N/A’’ indicates AUROC is not applicable (no probabilistic scores).  For sampling-based methods, we report mean $\pm$ std.\ over three trials (std.\ shown as subscript).}
\label{tab:Results-Internal-Mistral-7B-Instruct}
\end{table*}

\begin{table*}[t!]
\centering
\renewcommand{\arraystretch}{1.28}   
\resizebox{0.72\linewidth}{!}{%
\begin{tabular}{l|ccc|ccc}
\hline
\multirow{2}{*}{\textbf{Method}} &
\multicolumn{3}{c|}{\textbf{TriviaQA}} &
\multicolumn{3}{c}{\textbf{SQuAD}} \\ \cline{2-7}
& \textbf{Acc.} & \textbf{F1} & \textbf{AUROC} & \textbf{Acc.} & \textbf{F1} & \textbf{AUROC} \\ \hline
Prompt Llama3.2-1B-Instruct     &  50.9 &  67.0 &  N/A &  51.1 &  66.8 &  N/A \\
Prompt Llama3-8B-Instruct        &  77.0 &  76.6 &  N/A &  62.3 &  59.7 &  N/A \\
Prompt Mistral-7B-Instruct       &  51.8 &  7.7 &  N/A &  50.4 &  2.0 &  N/A \\ \hline
LastTokenEntropy                 &  53.6 &  33.6 &  53.8 &  51.0 &  15.9 &  52.4 \\
LogProb                          &  57.6 &  56.9 &  62.2 &  54.1 &  43.9 &  57.1 \\ \hline
pTrue Llama3.2-1B-Instruct      &  55.1$_{0.2}$ &  58.0$_{0.1}$ &  47.9$_{0.9}$ &  50.7$_{1.0}$ &  53.9$_{1.1}$ &  50.3$_{0.8}$ \\
pTrue Llama3-8B-Instruct         &  53.8$_{0.5}$ &  59.7$_{0.6}$ &  52.4$_{0.2}$ &  51.9$_{0.4}$ &  58.5$_{0.4}$ &  51.9$_{0.5}$ \\
pTrue Mistral-7B-Instruct        &  53.6$_{0.3}$ &  59.2$_{0.3}$ &  47.1$_{0.7}$ &  52.1$_{0.7}$ &  59.2$_{0.5}$ &  50.8$_{0.7}$ \\
LexicalSimilarity                &  74.6$_{0.2}$ &  73.9$_{0.6}$ &  72.8$_{0.1}$ &  63.5$_{0.1}$ &  68.7$_{0.6}$ &  67.5$_{0.1}$ \\
SemanticEntropy                  &  74.0$_{0.5}$ &  73.9$_{1.3}$ &  75.5$_{0.1}$ &  63.8$_{0.4}$ &  68.9$_{0.6}$ &  67.2$_{0.2}$ \\ \hline
\textbf{Semantic Volume (ours)} &  \textbf{75.2$_{0.1}$} &  \textbf{77.7$_{0.1}$} &  \textbf{75.6$_{0.8}$} &  \textbf{65.7$_{0.1}$} &  \textbf{71.4$_{0.1}$} &  \textbf{68.1$_{0.7}$} \\ \hline
\end{tabular}}
\caption{Response model: \texttt{Qwen2.5-1.5B-Instruct}. Internal uncertainty: Accuracy, F1, and AUROC on TriviaQA and SQuAD.  “N/A’’ indicates AUROC is not applicable (no probabilistic scores).  For sampling-based methods, we report mean $\pm$ std.\ over three trials (std.\ shown as subscript).}
\label{tab:Results-Internal-Qwen2.5-1.5B-Instruct}
\end{table*}

\begin{table*}[t!]
\centering
\renewcommand{\arraystretch}{1.28}   
\resizebox{0.72\linewidth}{!}{%
\begin{tabular}{l|ccc|ccc}
\hline
\multirow{2}{*}{\textbf{Method}} &
\multicolumn{3}{c|}{\textbf{TriviaQA}} &
\multicolumn{3}{c}{\textbf{SQuAD}} \\ \cline{2-7}
& \textbf{Acc.} & \textbf{F1} & \textbf{AUROC} & \textbf{Acc.} & \textbf{F1} & \textbf{AUROC} \\ \hline
Prompt Llama3.2-1B-Instruct     &  50.5 &  66.7 &  N/A &  50.6 &  66.6 &  N/A \\
Prompt Llama3-8B-Instruct        &  71.9 &  70.2 &  N/A &  61.1 &  56.3 &  N/A \\
Prompt Mistral-7B-Instruct       &  50.8 &  3.5 &  N/A &  50.2 &  1.1 &  N/A \\ \hline
LastTokenEntropy                 &  57.7 &  49.6 &  60.8 &  52.1 &  26.3 &  51.8 \\
LogProb                          &  56.2 &  38.0 &  62.3 &  55.9 &  48.8 &  59.4 \\ \hline
pTrue Llama3.2-1B-Instruct      &  51.5$_{0.6}$ &  56.2$_{0.4}$ &  50.4$_{0.6}$ &  49.3$_{0.7}$ &  52.6$_{0.6}$ &  49.9$_{0.7}$ \\
pTrue Llama3-8B-Instruct         &  50.9$_{0.2}$ &  57.7$_{0.2}$ &  49.7$_{0.7}$ &  50.2$_{0.8}$ &  56.8$_{0.8}$ &  50.3$_{0.4}$ \\
pTrue Mistral-7B-Instruct        &  52.7$_{0.8}$ &  57.5$_{0.7}$ &  51.5$_{1.0}$ &  51.6$_{0.6}$ &  57.2$_{0.4}$ &  50.3$_{0.5}$ \\
LexicalSimilarity                &  73.7$_{0.7}$ &  71.6$_{1.4}$ &  67.8$_{0.0}$ &  66.2$_{0.1}$ &  63.6$_{1.5}$ &  71.9$_{0.2}$ \\
SemanticEntropy                  &  73.4$_{0.6}$ &  71.5$_{1.6}$ &  69.8$_{0.2}$ &  64.1$_{0.4}$ &  62.5$_{0.9}$ &  69.1$_{0.1}$ \\ \hline
\textbf{Semantic Volume (ours)} &  \textbf{75.2$_{0.2}$} &  \textbf{77.3$_{0.1}$} &  \textbf{75.3$_{0.6}$} &  \textbf{66.7$_{0.4}$} &  \textbf{71.1$_{0.1}$} &  66.7$_{0.7}$ \\ \hline
\end{tabular}}
\caption{Response model: \texttt{Qwen3-8B}. Internal uncertainty: Accuracy, F1, and AUROC on TriviaQA and SQuAD.  “N/A’’ indicates AUROC is not applicable (no probabilistic scores).  For sampling-based methods, we report mean $\pm$ std.\ over three trials (std.\ shown as subscript).}
\label{tab:Results-Internal-Qwen3-8B}
\end{table*}

\begin{table*}[t!]
\centering
\renewcommand{\arraystretch}{1.28}   
\resizebox{0.72\linewidth}{!}{%
\begin{tabular}{l|ccc|ccc}
\hline
\multirow{2}{*}{\textbf{Method}} &
\multicolumn{3}{c|}{\textbf{TriviaQA}} &
\multicolumn{3}{c}{\textbf{SQuAD}} \\ \cline{2-7}
& \textbf{Acc.} & \textbf{F1} & \textbf{AUROC} & \textbf{Acc.} & \textbf{F1} & \textbf{AUROC} \\ \hline
Prompt Llama3.2-1B-Instruct     &  50.5 &  66.8 &  N/A &  50.6 &  66.5 &  N/A \\
Prompt Llama3-8B-Instruct        &  68.6 &  65.9 &  N/A &  59.5 &  53.5 &  N/A \\
Prompt Mistral-7B-Instruct       &  51.0 &  4.1 &  N/A &  50.3 &  1.8 &  N/A \\ \hline
LastTokenEntropy                 &  62.2 &  64.1 &  65.8 &  56.1 &  48.9 &  60.5 \\
LogProb                          &  58.2 &  49.8 &  62.5 &  58.0 &  62.8 &  60.3 \\ \hline
pTrue Llama3.2-1B-Instruct      &  52.5$_{0.3}$ &  56.1$_{0.4}$ &  51.1$_{0.9}$ &  48.9$_{0.4}$ &  51.4$_{0.4}$ &  51.2$_{0.2}$ \\
pTrue Llama3-8B-Instruct         &  51.2$_{0.8}$ &  56.9$_{0.8}$ &  49.8$_{0.5}$ &  50.2$_{0.6}$ &  56.3$_{0.6}$ &  49.8$_{0.3}$ \\
pTrue Mistral-7B-Instruct        &  54.3$_{0.3}$ &  58.2$_{0.4}$ &  53.2$_{0.5}$ &  51.3$_{0.3}$ &  56.7$_{0.3}$ &  51.6$_{0.2}$ \\
LexicalSimilarity                &  75.0$_{0.1}$ &  76.0$_{0.1}$ &  74.7$_{0.1}$ &  64.3$_{0.2}$ &  58.3$_{0.6}$ &  64.5$_{0.2}$ \\
SemanticEntropy                  &  74.4$_{0.2}$ &  74.6$_{0.7}$ &  74.2$_{0.1}$ &  63.1$_{0.5}$ &  57.6$_{1.5}$ &  65.5$_{0.2}$ \\ \hline
\textbf{Semantic Volume (ours)} &  \textbf{75.1$_{0.5}$} &  \textbf{76.9$_{0.6}$} &  \textbf{75.1$_{0.5}$} &  \textbf{66.1$_{0.6}$} &  \textbf{69.0$_{0.5}$} &  \textbf{66.0$_{0.6}$} \\ \hline
\end{tabular}}
\caption{Response model: \texttt{Qwen3-14B}. Internal uncertainty: Accuracy, F1, and AUROC on TriviaQA and SQuAD.  “N/A’’ indicates AUROC is not applicable (no probabilistic scores).  For sampling-based methods, we report mean $\pm$ std.\ over three trials (std.\ shown as subscript).}
\label{tab:Results-Internal-Qwen3-14B}
\end{table*}

\section{Hallucination Detection Pipeline using both External and Internal Uncertainty} \label{sec:Appendix-CompletePipeline}
In this paper, we conducted separate experiments on external uncertainty detection and internal uncertainty detection to demonstrate that the Semantic Volume is an effective method that can be generally applied to both tasks. In practice, these two tasks can be combined into a unified pipeline for hallucination detection (see Figure~\ref{fig:complete_pipeline}): First, we perform the internal uncertainty detection. If high uncertainty is detected in the response, then hallucination is likely to happen. Then in the second step, we check the external uncertainty: If the query is detected to be ambiguous, the LLM should ask a clarification question to the user. After clarification is provided or additional information is provided to resolve the ambiguity in the query, the LLM can then generate the answer to the query. On the other hand, if external uncertainty is ruled out, then the hallucination is likely caused by the internal lack of knowledge of the LLM. This can be addressed through various methods such as retrieval-augmented generation (RAG), reasoning-based approaches, or by leveraging stronger LLMs or human agents.

\begin{figure*}[t!]
    \centering
   \includegraphics[width= 0.9\linewidth]{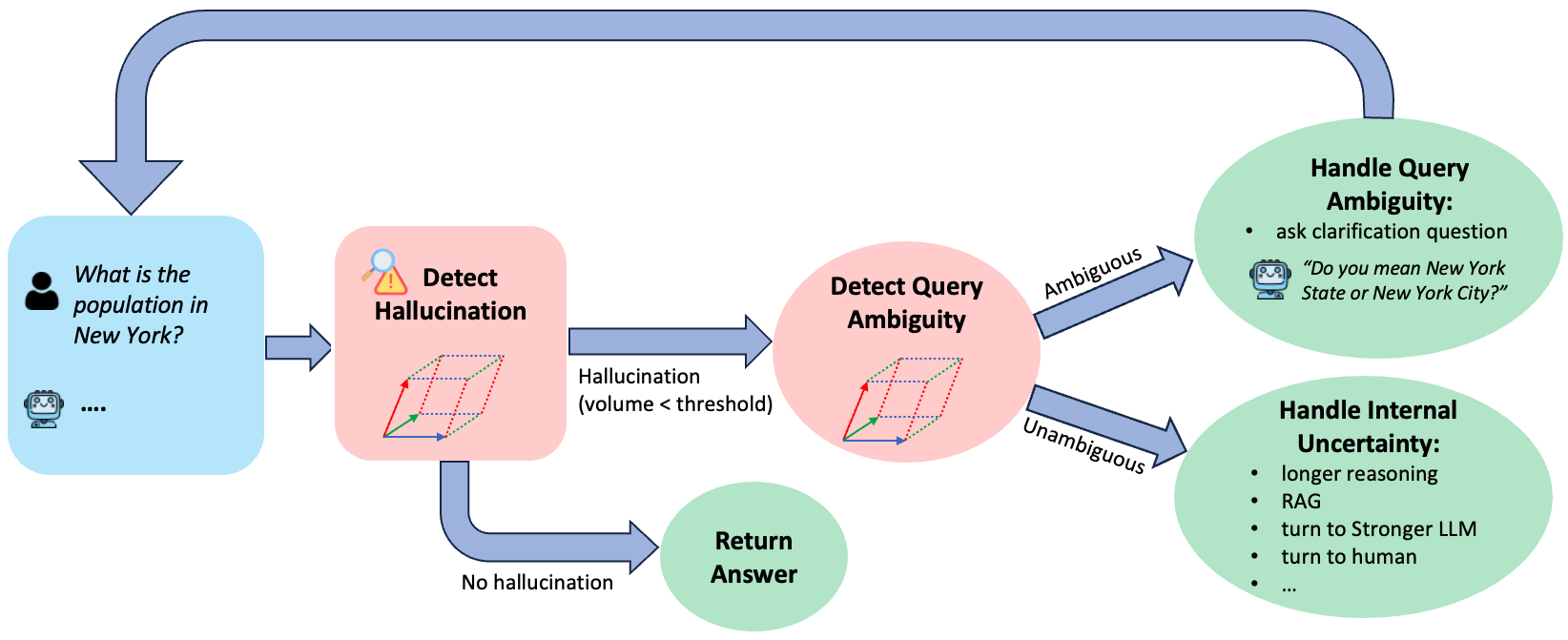}
    \caption{Complete pipeline for hallucination detection utilizing both external and internal uncertainty.}
    \label{fig:complete_pipeline}
\end{figure*}

Moreover, all of our experiments are set up as classification tasks for uncertainty detection, relying on a threshold-based decision. However, these uncertainty measures can also be applied to comparison or ranking tasks.

\section{Prompts} \label{sec:Appendix-Prompts}

\subsection{Prompt template to extend queries.}
\texttt{Provide a paraphrase of the following question with a contextual expansion, while maintaining its core meaning. The filled context information should be diverse  but must be concrete and specific (it cannot be a placeholder or a template). Only reply with the new version of the question and nothing else.}

\texttt{Question:}

\texttt{\{question\}}

\subsection{Prompt template for query ambiguity detection}

\texttt{Is the following question ambiguous? A question is ambiguous if it can be interpreted in multiple ways or has multiple possible answers. If the question is ambiguous, then reply `Yes', otherwise reply `No'. Only reply with `Yes' or `No' and nothing else.}

\texttt{Question:}

\texttt{\{question\}}

\section{Case study}\label{sec:Appendix-CaseStudy}
To further illustrate the distinct levels of uncertainty in queries and responses, we present a detailed case study in this section, providing concrete examples and their corresponding perturbations analyzed using our Semantic Volume method. For the queries in CLAMBER, we analyze three unambiguous examples with Semantic Volumes in the bottom 5\% percentile of the dataset, and three ambiguous examples with large Semantic Volumes (above the 95\% percentile). For demonstration purposes, we present five extended versions (perturbations) for each sample. We observe that the ambiguous queries exhibit high semantic dispersion (as captured by their large Semantic Volumes), while the unambiguous ones have extensions that are quite similar.

Similarly, we present three correct samples and three hallucinated samples with the lowest and highest Semantic Volume values, respectively, from the TriviaQA dataset. For each sample, we provide five additional sampled candidates and show that, for correct answers, the candidates are semantically equivalent, whereas for hallucinated answers, the candidates exhibit substantial diversity.

\subsection{External Uncertainty (Query Ambiguity)}
For the unambiguous queries, we observe that all extensions are quite similar and preserve the same core meaning. For each unambiguous query, we include the original clarification question from CLAMBER and provide our case analysis following each example.

\begin{tcolorbox}[colback=blue!5!white, colframe=blue!60!white, title=Examples of Unambiguous Queries, left=10mm, breakable]
\small 
\textbf{Example 1. Query:}  Who plays harry in the amazing spider man 2?

\bigskip
\textbf{Perturbations (extended queries):}
\begin{enumerate}
    \item In the 2014 superhero film "The Amazing Spider-Man 2," directed by Marc Webb and featuring Andrew Garfield as Peter Parker/Spider-Man, which actor portrays Harry Osborn, the childhood friend of Peter who eventually becomes the Green Goblin?
    \item In the 2014 superhero film "The Amazing Spider-Man 2," directed by Marc Webb and starring Andrew Garfield as Peter Parker/Spider-Man, which actor portrays Harry Osborn, the childhood friend of Peter who becomes the Green Goblin?
    \item In the 2014 superhero film "The Amazing Spider-Man 2," directed by Marc Webb and starring Andrew Garfield as Peter Parker/Spider-Man, which actor portrays the character of Harry Osborn, Peter's childhood friend who becomes the Green Goblin?
    \item In the 2014 superhero film "The Amazing Spider-Man 2," directed by Marc Webb and starring Andrew Garfield as Peter Parker, which actor portrays Harry Osborn, the childhood friend of Peter who eventually becomes the Green Goblin?
    \item In the 2014 superhero film "The Amazing Spider-Man 2," directed by Marc Webb and starring Andrew Garfield as Peter Parker, which actor portrays Harry Osborn, the childhood friend of Peter who becomes the Green Goblin?
\end{enumerate}
\bigskip

\hrule
\bigskip
\textbf{Example 2. Query:}  Who plays young william in a knight's tale?

\bigskip
\textbf{Perturbations (extended queries):}
\begin{enumerate}
    \item  In the 2001 medieval-themed sports drama "A Knight's Tale," starring Heath Ledger as the adult William Thatcher, which actor portrays the character of William as a child during the film's opening sequences set in 14th-century England?
    \item  In the 2001 medieval-themed sports drama "A Knight's Tale," starring Heath Ledger as the adult William Thatcher, which young actor portrays the childhood version of the protagonist during the film's opening scenes set in 14th-century England?
    \item  In the 2001 medieval-themed sports drama "A Knight's Tale," starring Heath Ledger as the adult William Thatcher, which actor portrays the character of William as a child during the film's opening scenes set in 14th-century England?
    \item  In the 2001 medieval-themed sports drama "A Knight's Tale," starring Heath Ledger as the adult William Thatcher, which actor portrays the childhood version of the protagonist during the film's opening sequences set in 14th-century England?
    \item  In the medieval-themed 2001 film "A Knight's Tale," starring Heath Ledger as the adult William Thatcher, which actor portrays the character of William as a child during the opening scenes set in 14th-century England?
\end{enumerate}
\bigskip

\hrule
\bigskip
\textbf{Example 3. Query:}  When did david cassidy release i think i love you?

\bigskip
\textbf{Perturbations (extended queries):}
\begin{enumerate}
    \item In what year did the teen idol and star of "The Partridge Family," David Cassidy, release his breakthrough hit single "I Think I Love You," which topped the Billboard Hot 100 chart and helped solidify his status as a 1970s pop music sensation?
    \item In what year did the teen idol and former star of "The Partridge Family," David Cassidy, release his chart-topping hit single "I Think I Love You," which became an anthem for young romance in the early 1970s and helped solidify his status as a pop sensation?
    \item In what year did the former teen idol and star of "The Partridge Family," David Cassidy, release his chart-topping hit single "I Think I Love You," which became an anthem for young romance in the early 1970s and helped solidify his status as a pop culture icon?
    \item In which year did the teen heartthrob and former star of "The Partridge Family," David Cassidy, release his chart-topping hit single "I Think I Love You," which became an anthem for young romance in the early 1970s and helped solidify his status as a pop idol?
    \item In which year did the teen idol and star of "The Partridge Family," David Cassidy, release his chart-topping hit single "I Think I Love You," which became an anthem for young romance in the early 1970s and helped solidify his status as a pop music sensation?

\end{enumerate}
\bigskip

\end{tcolorbox}

\begin{tcolorbox}[colback=blue!5!white, colframe=blue!60!white, title=Examples of Ambiguous Queries, left=10mm, breakable]
\small 
\textbf{Example 4. Query:}  What is the theme of Togetherness?

\bigskip
\textbf{Clarifying Question:}  Are you referring to the TV show or the concept of togetherness in general?

\bigskip
\textbf{Perturbations (extended queries):}
\begin{enumerate}
    \item How does the concept of unity and shared experiences manifest as the central motif in the 2015 Danish TV series "Togetherness," created by Mark and Jay Duplass, which explores the complexities of marriage, friendship, and personal growth among a group of middle-aged characters in Los Angeles?
    \item In what ways does the concept of unity and shared experience manifest itself as a central motif in "Togetherness," the HBO comedy-drama series created by Jay and Mark Duplass? How does the show explore the complexities of human relationships, particularly among its four main characters living in Los Angeles, as they navigate marriage, friendship, and personal growth?
    \item How does the theme of togetherness manifest in human relationships, whether it's in the tight-knit bonds of family gatherings during Thanksgiving dinner, the camaraderie among teammates celebrating a championship victory, or the collective spirit of strangers coming together to help their community rebuild after a natural disaster?
    \item How does the concept of unity and collective experience manifest as a central motif in the novel "Togetherness" by acclaimed author Sarah Chen, which explores the interconnected lives of residents in a small coastal town facing environmental challenges?
    \item How does the central message of unity and human connection manifest throughout the album "Togetherness" by jazz trumpeter Roy Hargrove, particularly in tracks like "Trust" and "Strasbourg/St. Denis" which emphasize collective improvisation and musical dialogue between band members?
\end{enumerate}
\bigskip

\textbf{Analysis:} The extended queries above illustrate distinct interpretations of the ambiguous term ``Togetherness.'' The first two extensions interpret it as the HBO television series Togetherness (2015) by Mark and Jay Duplass. The third treats it as a general conceptual inquiry into human connection, unity, and solidarity. The fourth interprets ``Togetherness'' as the title of a novel, while the final extension refers to a music album--Togetherness by jazz trumpeter Roy Hargrove. These diverse extensions highlight the semantic ambiguity of the original query, as well as the high variability in meaning captured by the Semantic Volume method.

\bigskip
\hrule
\bigskip
\textbf{Example 5. Query:}  What do you call a ten sided shape?

\bigskip
\textbf{Clarifying Question:} Which one: 2d shape, or 3d shape?

\bigskip
\textbf{Perturbations (extended queries):}
\begin{enumerate}
    \item What is the geometric term for a polygon with ten sides and ten angles, such as the shape found in the British 50 pence coin or the logo of the Detroit Red Wings hockey team?
    \item In geometry class, while studying regular polygons, the teacher introduced various shapes. After discussing triangles, squares, pentagons, and octagons, she asked the students: "Moving on to more complex figures, what is the proper term for a polygon with exactly ten edges and ten vertices?"
    \item  In the field of geometry, particularly when studying polygons and their properties, what is the specific term used to describe a two-dimensional figure with exactly ten straight sides and ten vertices? This shape is commonly encountered in various applications, from architectural designs to board game dice, and understanding its proper nomenclature is essential for mathematical discussions and practical applications.
    \item  What is the mathematical term for a polygon with ten edges and ten vertices, like the shape of the US Pentagon building if it had five additional sides or the cross-section of some crystalline minerals found in nature?
    \item  What is the geometric term for a polygon with ten edges and ten vertices, like the profile of a traditional Chinese pagoda or the cross-section of certain crystals found in mineralogy?
\end{enumerate}
\bigskip

\textbf{Analysis:} While all extensions refer to ten-sided geometric figures, their interpretations vary in subtle ways. The first three primarily describe 2D polygons, consistently referencing shapes with ten sides and vertices in mathematical or educational contexts. The last two, however, introduce 3D analogies, such as architectural structures or crystal cross-sections. This variation highlights the ambiguity in the original query.

\bigskip
\hrule
\bigskip
\textbf{Example 6. Query:}  What is the cure for Cured?

\bigskip
\textbf{Clarifying Question:} Are you referring to a medical condition or a song?

\bigskip
\textbf{Perturbations (extended queries):}
\begin{enumerate}
    \item How can we effectively treat and reverse the effects of the fictional disease known as "Cured," which paradoxically causes illness despite its name, in the dystopian world of Lauren DeStefano's Chemical Garden trilogy?
    \item What medical treatment or intervention has been discovered to effectively eliminate or reverse the effects of Cured syndrome, a rare genetic disorder characterized by progressive muscle weakness and cognitive decline, which was first identified in 2008 by Dr. Elena Cured at Johns Hopkins University?
    \item After watching the 2020 psychological horror film "Cured" directed by Sean Ellis, where a woman discovers she has a mysterious condition that makes her immune to diseases but also causes violent outbursts, what medical treatment or intervention ultimately resolves her affliction?
    \item How can we treat or reverse "Cured," the psychological horror game released in 2016 that follows a patient trapped in an abandoned mental asylum who must escape while uncovering dark secrets about illegal experiments and supernatural phenomena?
    \item What is the medical treatment or remedy for the ironic condition known as Cured Syndrome - a rare autoimmune disorder first identified in Uppsala, Sweden in 1982 that causes patients to experience persistent sensations of perfect health despite having serious underlying symptoms?
\end{enumerate}
\bigskip

\textbf{Analysis:} The term ``Cured'' is interpreted in five distinct ways: as a fictional disease, a genetic disorder, a psychological horror film, a video game, and an ironically named medical condition. These varied interpretations highlight the ambiguity of the original query and the broad semantic spread, which is captured by our Semantic Volume method.

\bigskip

\end{tcolorbox}

\subsection{Internal Uncertainty (Response Uncertainty)}
In the cases where the responses exhibit low Semantic Volumes (the first three examples), the answers are correct and the sampled candidates are nearly identical. This consistency reflects the model’s strong confidence and low uncertainty. In contrast, for cases with high Semantic Volumes (last three examples), the sampled candidates show substantial diversity.

\begin{tcolorbox}[colback=blue!5!white, colframe=blue!60!white, title=Examples of Correct Responses, left=10mm, breakable]
\small 
\textbf{Example 1. Query:}  The Ebro River is in which country?

\textbf{Response:} Spain

\textbf{Ground Truth Answer:} Spain

\bigskip
\textbf{Perturbations (candidate responses):}
\begin{enumerate}
    \item Spain
    \item Spain
    \item Spain
    \item Spain
    \item Spain
\end{enumerate}
\bigskip

\hrule
\bigskip
\textbf{Example 2. Query:} Gorgonzola cheese is from which country?

\textbf{Response:} Italy

\textbf{Ground Truth Answer:} Italy

\bigskip
\textbf{Perturbations (candidate responses):}
\begin{enumerate}
    \item Italy
    \item Italy
    \item Italy
    \item Italy
    \item Italy
\end{enumerate}
\bigskip

\hrule
\bigskip
\textbf{Example 3. Query:} Tipperary is in which European country??

\textbf{Response:} Ireland

\textbf{Ground Truth Answer:} Ireland

\bigskip
\textbf{Perturbations (candidate responses):}
\begin{enumerate}
    \item Ireland
    \item Ireland
    \item Ireland
    \item Republic of Ireland
    \item Ireland
\end{enumerate}
\bigskip

\end{tcolorbox}

\begin{tcolorbox}[colback=blue!5!white, colframe=blue!60!white, title=Examples of Hallucinated Responses, left=10mm, breakable]
\small 
\textbf{Example 4. Query:}  In Top Cat, who was the voice of Choo Choo?

\textbf{Response:} John Kassir

\textbf{Ground Truth Answer:} Marvin Kaplan

\bigskip
\textbf{Perturbations (candidate responses):}
\begin{enumerate}
    \item Ricky Gervais
    \item Bob Denver
    \item Steve O’Brien
    \item Mark Linn-Baker 
    \item Seth Carwell
\end{enumerate}
\bigskip

\hrule
\bigskip
\textbf{Example 5. Query:} Who was the artist on the TV series Vision On?

\textbf{Response:} David Hasselhoff

\textbf{Ground Truth Answer:} Tony Hart

\bigskip
\textbf{Perturbations (candidate responses):}
\begin{enumerate}
    \item Michael Welch
    \item Brian Thompson
    \item Burt Reynolds
    \item Nicole Boyer
    \item Jim Parsons
\end{enumerate}
\bigskip

\hrule
\bigskip
\textbf{Example 6. Query:} Who wrote the novel The Go Between?

\textbf{Response:} Edith Pattou

\textbf{Ground Truth Answer:} L. P. Hartley

\bigskip
\textbf{Perturbations (candidate responses):}
\begin{enumerate}
    \item Jennifer Weiner
    \item Jojo Moyes
    \item John Irving
    \item Margaret Mitchell
    \item Alison Bechdel
\end{enumerate}
\bigskip

\end{tcolorbox}

\section{Variation: Query Extension Model} \label{sec:Appendix-Variation-ExtensionModels}
In the external uncertainty task, we investigate the impact of the query extension model by replacing the default \texttt{Claude3.7-Sonnet} with smaller models and reporting their corresponding results. Specifically, we explore the following query extension models:
\begin{itemize}
    \item \texttt{Claude3.7-\allowbreak Sonnet} (Tables~\ref{tab:Results-External}),
    \item \texttt{Qwen2.5-\allowbreak 1.5B-\allowbreak Instruct} (Table~\ref{tab:Results-External-Qwen2.5-1.5B-Instruct}),
    \item \texttt{Llama3-8B-Instruct} (Table~\ref{tab:Results-External-Llama3-8B-Instruct}),
    \item \texttt{Qwen3-\allowbreak 8B} (Table~\ref{tab:Results-External-Qwen3-8B}).
\end{itemize}
In addition, for baseline methods that rely on token probabilities (e.g., \textit{Last Token Entropy} and \textit{Log Probabilities}), we enhance robustness by evaluating each method with multiple underlying models for token extraction, including \texttt{Qwen2.5-\allowbreak 1.5B-\allowbreak Instruct}, \texttt{Llama3.2-\allowbreak 1B-\allowbreak Instruct}, and \texttt{Llama3-\allowbreak 8B-\allowbreak Instruct}.

Across all tested query extension models, our \textit{Semantic Volume} method consistently outperforms the baselines in general. We also observe that stronger query extension models, such as \texttt{Claude3.7-Sonnet}, lead to generally improved performance for all sampling-based approaches.

\begin{table*}[t!]
\centering
\renewcommand{\arraystretch}{1.09}   
\resizebox{0.65\textwidth}{!}{%
\begin{tabular}{l|cc|cc}
\hline
\multirow{2}{*}{\textbf{Method}} &
\multicolumn{2}{c|}{\textbf{CLAMBER}} &
\multicolumn{2}{c}{\textbf{AmbigQA}} \\ \cline{2-5}
& \textbf{Acc.} & \textbf{F1} & \textbf{Acc.} & \textbf{F1} \\ \hline
Vicuna-13B (zero-shot)                 & 50.6 & 39.9 & 45.1 & 19.5 \\
Llama2-13B-Instruct (zero-shot)        & 45.6 & 43.6 & 46.2 & 22.4 \\
Llama2-70B-Instruct (zero-shot)        & 50.3 & 34.2 & 49.4 & 64.5 \\
Llama3.2-3B-Instruct (zero-shot)       & 51.5 & 37.7 & 49.3 & 33.8 \\
ChatGPT (zero-shot)                    & 54.3 & 53.4 & 54.8 & 52.8 \\
ChatGPT (few-shot)                     & 51.6 & 49.2 & 54.9 & 53.1 \\
ChatGPT (zero-shot + CoT)              & 57.3 & 56.9 & 54.2 & 55.1 \\
ChatGPT (few-shot + CoT)               & 53.6 & 51.4 & 54.0 & 50.8 \\ \hline
LastTokenEntropy (Qwen2.5-1.5B-Instruct)     &  49.4 &  66.0 &  51.7 &  60.4 \\
LastTokenEntropy (Llama3.2-1B-Instruct)     &  49.5 &  66.1 &  50.7 &  64.7 \\
LastTokenEntropy (Llama3-8B-Instruct)  &  49.4 &  66.1 &  51.5 &  64.9 \\
LogProb (Qwen2.5-1.5B-Instruct)              &  49.4 &  64.9 &  55.9 &  65.2 \\
LogProb (Llama3.2-1B-Instruct)              &  48.9 &  63.8 &  52.7 &  64.2 \\
LogProb (Llama3-8B-Instruct)           &  49.2 &  63.6 &  53.1 &  65.3 \\ \hline
pTrue (Llama3-8B-Instruct)             &  49.7$_{0.3}$ &  41.7$_{0.4}$ &  53.3$_{0.6}$ &  52.4$_{0.9}$ \\
pTrue (Mistral-7B-Instruct)                  &  50.0$_{0.0}$ &  0.0$_{0.0}$ &  50.0$_{0.0}$ &  0.0$_{0.0}$ \\
Lexical Similarity                           &  51.1$_{0.3}$ &  57.0$_{0.6}$ &  55.0$_{0.1}$ &  21.3$_{0.3}$ \\
Semantic Entropy                             &  49.4$_{0.4}$ &  48.7$_{1.9}$ &  49.9$_{0.6}$ &  64.6$_{0.8}$ \\ \hline
\textbf{Semantic Volume (ours)}              &  \textbf{57.8$_{0.7}$} &  \textbf{67.1$_{0.5}$} &  \textbf{57.5$_{0.8}$} &  \textbf{67.0$_{0.5}$} \\ \hline
\end{tabular}}
\caption{Query Extension Model: \texttt{Qwen2.5-1.5B-Instruct}. External uncertainty: Accuracy and F1 on CLAMBER and AmbigQA. Sampling‐based methods report mean $\pm$ std.\ over three trials (std.\ as subscript).}
\label{tab:Results-External-Qwen2.5-1.5B-Instruct}
\end{table*}

\begin{table*}[t!]
\centering
\renewcommand{\arraystretch}{1.09}   
\resizebox{0.65\textwidth}{!}{%
\begin{tabular}{l|cc|cc}
\hline
\multirow{2}{*}{\textbf{Method}} &
\multicolumn{2}{c|}{\textbf{CLAMBER}} &
\multicolumn{2}{c}{\textbf{AmbigQA}} \\ \cline{2-5}
& \textbf{Acc.} & \textbf{F1} & \textbf{Acc.} & \textbf{F1} \\ \hline
Vicuna-13B (zero-shot)                 & 50.6 & 39.9 & 45.1 & 19.5 \\
Llama2-13B-Instruct (zero-shot)        & 45.6 & 43.6 & 46.2 & 22.4 \\
Llama2-70B-Instruct (zero-shot)        & 50.3 & 34.2 & 49.4 & 64.5 \\
Llama3.2-3B-Instruct (zero-shot)       & 51.5 & 37.7 & 49.3 & 33.8 \\
ChatGPT (zero-shot)                    & 54.3 & 53.4 & 54.8 & 52.8 \\
ChatGPT (few-shot)                     & 51.6 & 49.2 & 54.9 & 53.1 \\
ChatGPT (zero-shot + CoT)              & 57.3 & 56.9 & 54.2 & 55.1 \\
ChatGPT (few-shot + CoT)               & 53.6 & 51.4 & 54.0 & 50.8 \\ \hline
LastTokenEntropy (Qwen2.5-1.5B-Instruct)     &  49.4 &  66.0 &  51.7 &  60.4 \\
LastTokenEntropy (Llama3.2-1B-Instruct)     &  49.5 &  66.1 &  50.7 &  64.7 \\
LastTokenEntropy (Llama3-8B-Instruct)  &  49.4 &  66.1 &  51.5 &  64.9 \\
LogProb (Qwen2.5-1.5B-Instruct)              &  49.4 &  64.9 &  55.9 &  65.2 \\
LogProb (Llama3.2-1B-Instruct)              &  48.9 &  63.8 &  52.7 &  64.2 \\
LogProb (Llama3-8B-Instruct)           &  49.2 &  63.6 &  53.1 &  65.3 \\ \hline
pTrue (Llama3-8B-Instruct)             &  49.1$_{0.4}$ &  42.0$_{0.3}$ &  51.9$_{0.4}$ &  54.7$_{0.4}$ \\
pTrue (Mistral-7B-Instruct)                  &  50.0$_{0.0}$ &  0.0$_{0.0}$ &  50.0$_{0.0}$ &  0.0$_{0.0}$ \\
Lexical Similarity                           &  51.2$_{0.2}$ &  62.5$_{0.5}$ &  49.9$_{0.3}$ &  66.0$_{0.3}$ \\
Semantic Entropy                             &  49.7$_{0.2}$ &  54.1$_{5.2}$ &  48.9$_{0.5}$ &  61.2$_{2.2}$ \\ \hline
\textbf{Semantic Volume (ours)}              &  \textbf{57.4$_{0.6}$} &  \textbf{66.3$_{0.5}$} &  54.8$_{0.6}$ &  \textbf{67.2$_{0.6}$} \\ \hline
\end{tabular}}
\caption{Query Extension Model: \texttt{Llama3-8B-Instruct}. External uncertainty: Accuracy and F1 on CLAMBER and AmbigQA. Sampling‐based methods report mean $\pm$ std.\ over three trials (std.\ as subscript).}
\label{tab:Results-External-Llama3-8B-Instruct}
\end{table*}

\begin{table*}[t!]
\centering
\renewcommand{\arraystretch}{1.09}   
\resizebox{0.65\textwidth}{!}{%
\begin{tabular}{l|cc|cc}
\hline
\multirow{2}{*}{\textbf{Method}} &
\multicolumn{2}{c|}{\textbf{CLAMBER}} &
\multicolumn{2}{c}{\textbf{AmbigQA}} \\ \cline{2-5}
& \textbf{Acc.} & \textbf{F1} & \textbf{Acc.} & \textbf{F1} \\ \hline
Vicuna-13B (zero-shot)                 & 50.6 & 39.9 & 45.1 & 19.5 \\
Llama2-13B-Instruct (zero-shot)        & 45.6 & 43.6 & 46.2 & 22.4 \\
Llama2-70B-Instruct (zero-shot)        & 50.3 & 34.2 & 49.4 & 64.5 \\
Llama3.2-3B-Instruct (zero-shot)       & 51.5 & 37.7 & 49.3 & 33.8 \\
ChatGPT (zero-shot)                    & 54.3 & 53.4 & 54.8 & 52.8 \\
ChatGPT (few-shot)                     & 51.6 & 49.2 & 54.9 & 53.1 \\
ChatGPT (zero-shot + CoT)              & 57.3 & 56.9 & 54.2 & 55.1 \\
ChatGPT (few-shot + CoT)               & 53.6 & 51.4 & 54.0 & 50.8 \\ \hline
LastTokenEntropy (Qwen2.5-1.5B-Instruct)     &  49.4 &  66.0 &  51.7 &  60.4 \\
LastTokenEntropy (Llama3.2-1B-Instruct)     &  49.5 &  66.1 &  50.7 &  64.7 \\
LastTokenEntropy (Llama3-8B-Instruct)  &  49.4 &  66.1 &  51.5 &  64.9 \\
LogProb (Qwen2.5-1.5B-Instruct)              &  49.4 &  64.9 &  55.9 &  65.2 \\
LogProb (Llama3.2-1B-Instruct)              &  48.9 &  63.8 &  52.7 &  64.2 \\
LogProb (Llama3-8B-Instruct)           &  49.2 &  63.6 &  53.1 &  65.3 \\ \hline
pTrue (Llama-3-8B-Instruct)             &  49.1$_{0.3}$ &  29.1$_{0.6}$ &  49.8$_{0.3}$ &  39.6$_{0.4}$ \\
pTrue (Mistral-7B-Instruct)                  &  50.0$_{0.0}$ &  0.0$_{0.0}$ &  50.0$_{0.0}$ &  0.0$_{0.0}$ \\
Lexical Similarity                           &  51.9$_{0.1}$ &  51.7$_{3.8}$ &  51.3$_{0.1}$ &  60.3$_{1.5}$ \\
Semantic Entropy                             &  49.3$_{0.9}$ &  44.0$_{9.7}$ &  49.9$_{0.8}$ &  43.5$_{5.5}$ \\ \hline
\textbf{Semantic Volume (ours)}              &  55.5$_{0.8}$ &  \textbf{66.5$_{0.5}$} &  \textbf{57.2$_{0.7}$} &  \textbf{67.5$_{0.5}$} \\ \hline
\end{tabular}}
\caption{Query Extension Model: \texttt{Qwen3-8B}. External uncertainty: Accuracy and F1 on CLAMBER and AmbigQA. Sampling‐based methods report mean $\pm$ std.\ over three trials (std.\ as subscript).}
\label{tab:Results-External-Qwen3-8B}
\end{table*}

\section{More Discussion}

To support broader adoption and ensure practical robustness, we discuss several core aspects of our design choices.

\paragraph{Coverage of Confident Predictions.}
Semantic Volume is designed to quantify model uncertainty in cases where elevated dispersion reflects ambiguity or error. A separate class of failure modes involves \emph{confidently incorrect responses}, where the model produces erroneous answers with high certainty. This is fundamentally a different problem: for example, if a model has been trained on faulty data (such as memorizing ``1+1=3'') then from its perspective, that answer is correct. In such cases, probing the model's token probabilities, hidden states, or sampled responses will not reveal any anomaly, because the internal uncertainty is low by design. Addressing these failures requires techniques that operate \emph{outside} the model, such as factual verification, training data auditing, or external consistency checks. As such, confidently wrong responses lie outside the scope of our current work and require complementary solutions beyond intrinsic uncertainty measures.

\paragraph{Efficiency in Sampling and Computation.}
Our approach relies on sampling perturbed queries or responses (typically $n = 20$), in line with other sampling-based metrics. The computation of the semantic volume itself is highly efficient, requiring only seconds per dataset split and scaling linearly with the number of examples. Appendix~\ref{sec:Appendix-Variation-n} shows that even with $n=10$, performance remains competitive. This enables a practical accuracy–cost trade-off. Moreover, our method does not require log-probabilities or gradient access, making it compatible with both open and closed LLM APIs.

\paragraph{Calibrated Decision Thresholds.}
To convert continuous scores into binary predictions (e.g. ambiguous or not), we fit a threshold using a small, labeled subset. This process is non-iterative, uses a closed-form objective (Proposition~\ref{prop:OptimalTauFormula}), and requires only a modest annotation effort. Our experiments show that a threshold derived from just 100 examples is sufficient and generalizes well across datasets and response models. This enhances usability in low-label or cross-domain settings.

\paragraph{Distributional Assumptions in Theory.}
Our theoretical connection between semantic volume and differential entropy relies on the assumption that embedding perturbations are approximately Gaussian after PCA. This provides a useful analytical interpretation, but is not strictly required for the log-determinant to act as a reliable dispersion measure. As shown in Appendix~\ref{sec:Appendix-Proofs}, PCA-normalized embeddings pass normality checks in the majority of cases, especially for query perturbations, supporting the practical soundness of the theoretical intuition.

\paragraph{Factuality Labeling Strategy.}
For hallucination detection, we use a ROUGE-L threshold to identify deviations from ground truth, following prior work. Although this heuristic may underestimate paraphrased but correct responses, a blind human validation on 1000 samples confirms over 95\% agreement with our automatic labeling (Section~\ref{sec:Experiment-Internal}). This indicates strong alignment between the heuristic and human judgment, offering a high-quality supervision signal at scale.

\end{document}